\documentclass[11pt,letterpaper]{article}
\usepackage[top=1in,bottom=1in,left=1in,right=1in]{geometry}

\usepackage[utf8]{inputenc}
\usepackage{hyperref}
\usepackage{url}
\usepackage{booktabs}
\usepackage{amsfonts}
\usepackage{nicefrac}
\usepackage{xcolor}

\usepackage{amsmath, amssymb}
\usepackage{amsthm}
\usepackage{graphicx}

\newtheorem{theorem}{Theorem}
\newtheorem{corollary}{Corollary}
\newtheorem{proposition}{Proposition}
\newtheorem{lemma}{Lemma}
\newtheorem{definition}{Definition}

\usepackage{natbib}
\usepackage{caption}
\usepackage{authblk}

\title{Support Recovery in Sparse PCA with Incomplete Data}

\author[1]{Hanbyul Lee}
\author[1]{Qifan Song}
\author[2]{Jean Honorio}
\affil[1]{Department of Statistics, Purdue University}
\affil[2]{Department of Computer Science, Purdue University}

\date{\today}

\begin{document}

\maketitle

\begin{abstract}
  
We study a practical algorithm for sparse principal component analysis (PCA) of incomplete and noisy data.
Our algorithm is based on the semidefinite program (SDP) relaxation of the non-convex $l_1$-regularized PCA problem.
We provide theoretical and experimental evidence that SDP enables us to exactly recover the true support of the sparse leading eigenvector of the unknown true matrix, despite only observing an incomplete (missing uniformly at random) and noisy version of it.
We derive sufficient conditions for exact recovery, which involve matrix incoherence, the spectral gap between the largest and second-largest eigenvalues, the observation probability and the noise variance.
We validate our theoretical results with incomplete synthetic data, and show encouraging and meaningful results on a gene expression dataset.    
  
\end{abstract}

\section{Introduction}
\label{sec:introduction}

Principal component analysis (PCA) is one of the most popular methods to reduce data dimension which is widely used in various applications including genetics, image processing, engineering, and many others. 
However, standard PCA is usually not preferred when principal components depend on only a small number of variables, because it provides dense vectors as a solution which degrades interpretability of the result. 
This can be worse especially in the high-dimensional setting where the solution of standard PCA is inconsistent as addressed in several works \citep{paul2007asymptotics, nadler2008finite, johnstone2009consistency}. 
To solve the inconsistency issue and improve interpretability, {\it sparse PCA} has been proposed, which enforces sparsity in the PCA solution so that dimension reduction and variable selection can be simultaneously performed. 
Theoretical and algorithmic researches on sparse PCA have been actively conducted over the past few years \citep{zou2006sparse, amini2008high, journee2010generalized, ma2013sparse, lei2015sparsistency, berk2019certifiably, 
richtarik2021alternating}.

In this paper, we consider a special situation where the data to which sparse PCA is applied are not completely observed, but partially missing.
Missing data frequently occurs in a wide range of machine learning problems, where sparse PCA is no exception. 
There are various reasons and situations where data becomes incomplete, such as failures of hardware, high expenses of sampling, and preserving privacy.
One concrete example is the analysis of single-cell RNA sequence (scRNA-seq) data \citep{park2019sparse}, where the cells are divided into several distinct types which can be characterized with only a small number of genes among tens of thousands of genes. 
Sparse PCA can be effectively utilized here to reduce the dimension (from numerous cells to a few cell types) and to select a small number of genes that affect the reduced data. 
However, since scRNA-seq data usually have many missing values due to technical and sampling issues, the existing sparse PCA theory and method designed for fully observed data cannot be directly applied, and new methodology and theory are in demand.

Despite the need for theoretical research and algorithmic development of sparse PCA for incomplete data, there have not been many studies yet. 
\cite{lounici2013sparse} and \cite{kundu2015approximating} considered two different optimization objectives for sparse PCA on incomplete data, which impose $l_1$ regularization and $l_0$ constraint on the classic PCA loss function using a (bias-corrected) incomplete matrix, respectively.
It was shown that the solution of each problem has a non-trivial error bound under certain conditions, but the optimization problems they considered are either nonconvex or NP-hard, and thus theoretical studies of computational feasible algorithms are still lacking. 
More recently, \cite{park2019sparse} proposed a computationally tractable two-step algorithm based on matrix factorization and completion, but its first step is an iterative algorithm that requires singular value decomposition in every iteration, which incurs a lot of cost in memory and time under a high-dimensional setting. 

With this motivation, we suggest a computational friendly convex optimization problem via a semidefinite relaxation of the $l_1$ regularized PCA, to solve the sparse PCA on incomplete data.
We note that very efficient scalable SDP solvers exist in practice \citep{yurtsever2021scalable}.
We assume that the unknown true matrix $\pmb{M}^*\in\mathbb{R}^{d\times d}$ is symmetric and has a sparse leading eigenvector $\pmb{u}_1$. 
Our goal is to exactly recover the support of this sparse leading eigenvector, i.e., to find the set $J$ correctly where $J = supp(\pmb{u}_1) = \{i : u_{1,i}\neq 0\}$. 
Given a noisy observation $\pmb{M}$ for the unknown true matrix $\pmb{M}^*$, it is intuitive to consider imposing a regularization term on the PCA quadratic loss that aims to find the first principal component. 
When using the $l_1$ regularizer, the optimization problem can be written as:
\begin{equation*}
\hat{\pmb{x}} = \underset{\pmb{x}^\top \pmb{x} = 1}{\arg\max}~ \pmb{x}^\top \pmb{M}\pmb{x} - \rho \|\pmb{x}\|_1^2.
\end{equation*}
Hence, $J$ is estimated with $supp(\hat{\pmb{x}})$.
However, this intuitively appealing objective is nonconvex and very difficult to solve, so the following semidefinite relaxation can be considered as an alternative:
\begin{equation*}
\hat{\pmb{X}} = \underset{\pmb{X}\succeq 0 \text{ and } tr(\pmb{X}) = 1}{\arg\max}~ 
\langle \pmb{M}, \pmb{X} \rangle - \rho \|\pmb{X}\|_{1,1}.
\end{equation*}
By letting $\pmb{X} = \pmb{x}\pmb{x}^\top$, the equivalence of the above two objective functions can be easily justified.
Since $supp(\pmb{x}) = supp(diag(\pmb{x}\pmb{x}^\top))$, we estimate the support $J$ by $\hat{J} = supp(diag(\hat{\pmb{X}}))$ in the semidefinite problem.
This kind of relaxation has been studied by \cite{d2004direct} and \cite{lei2015sparsistency}, but their works were limited to complete data.
Surprisingly, without any additional modifications on the relaxation problem such as using matrix factorization or matrix completion, we show that it is possible to exactly recover true support $J$ with the above semidefinite program itself when $\pmb{M}$ is an incomplete observation.
Our main contribution is to prove this claim theoretically and experimentally.

In Section \ref{sec:main_results}, we provide theoretical justification (i.e., Theorem \ref{thm_main}) that we can exactly recover the true support $J$ with high probability by obtaining a unique solution of the semidefinite problem, under proper conditions.
The conditions involve matrix coherence parameters, the spectral gap between the largest and second-largest eigenvalues of the true matrix, the observation probability and the noise variance, which are discussed in detail in Corollaries \ref{cor_1} and \ref{cor_2}.
Specifically, we show that the sample complexity is related to the matrix coherence parameters as well as the matrix dimension $d$ and the support size $s$.
We prove that the observation probability $p$ has the bound of $p = \omega\Big(\frac{1}{d^{-1} + 1}\Big)$ in the worst scenario in terms of the matrix coherence, while 
it has a smaller lower bound $p=\omega\Big(\frac{1}{(\log s)^{-1} + 1}\Big)$ in the best scenario.
In Section \ref{sec:numerical_results}, we provide experimental results on incomplete synthetic datasets and a gene expression dataset.
The experiment on the synthetic datasets validate our theoretical results,
and the experiment on the gene expression dataset gives us a consistent result with prior studies.

\section{Preliminaries}
\label{sec:preliminaries}

\subsection{Notation}
\label{subsec:notation}

We first introduce the notations used throughout the paper.
Matrices are bold capital, vectors are bold lowercase and scalars or entries are not bold.
For any positive integer $n$, we denote $[n]:=\{1,\dots,n\}$.
For any vector $\pmb{a} \in \mathbb{R}^d$ and index set $J \subseteq [d]$, $\pmb{a}_J$ denotes the $|J|$-dimensional vector consisting of the entries of $\pmb{a}$ in $J$.
For any matrix $\pmb{A} \in \mathbb{R}^{d_1 \times d_2}$ and index sets $J_1 \subseteq [d_1]$ and $J_2 \subseteq [d_2]$, $\pmb{A}_{J_1,J_2}$ and $\pmb{A}_{J_1,:} (\pmb{A}_{:,J_2})$ denote the $|J_1|\times|J_2|$ sub-matrix of $\pmb{A}$ consisting of rows in $J_1$ and columns in $J_2$, and the $|J_1|\times d_2$ ($d_1 \times |J_2|$) sub-matrix of $\pmb{A}$ consisting of rows in $J_1$ (columns in $J_2$), respectively.
$\|\pmb{a}\|_1$, $\|\pmb{a}\|_2$ and $\|\pmb{a}\|_\infty$ represent the $l_1$ norm, $l_2$ norm and maximum norm of a vector $\pmb{a}$, respectively.
$\{ \pmb{e}_i ~:~ i\in [d] \}$ indicates the standard basis of $\mathbb{R}^d$.

A variety of norms on matrices will be used: we denote by $\|\pmb{A}\|_2$ the spectral norm and by $\|\pmb{A}\|_F$ the Frobenius norm of a matrix $\pmb{A}$.
We let $\|\pmb{A}\|_{1,1} = \sum_{i\in[d_1], j\in[d_2]}|A_{i,j}|$, 
$\|\pmb{A}\|_{\max} = \|\pmb{A}\|_{\infty,\infty} = \max_{i\in[d_1], j\in[d_2]}|A_{i,j}|$,
$\|\pmb{A}\|_{2,\infty} = \max_{j\in[d_2]} \|\pmb{A}_{:,j}\|_2$
and $\|\pmb{A}\|_{1,\infty} = \max_{j\in[d_2]} \|\pmb{A}_{:,j}\|_1$ represent the $l_{1,1}$ norm, the entrywise $l_{\infty}$ norm, the $l_{2,\infty}$ norm and the $l_{1,\infty}$ norm of a matrix $\pmb{A}$, respectively.
The trace of $\pmb{A}$ is denoted $tr(\pmb{A})$, and the matrix inner product of $\pmb{A}$ and $\pmb{B}$ is denoted $\langle \pmb{A}, \pmb{B} \rangle$.
Also, $\sigma_i(\pmb{A})$ and $\lambda_i(\pmb{A})$ represent the $i$th largest singular value and the $i$th largest eigenvalue of $\pmb{A}$, respectively.

The notation $C,C_1,\dots,c,c_1,\dots$ denote positive constants whose values may change from line to line.
The notation $f(x) = o(g(x))$ or $f(x) \ll g(x)$ means $\lim_{x\rightarrow\infty}f(x)/g(x) = 0$;
$f(x) = \omega(g(x))$ or $f(x) \gg g(x)$ means $\lim_{x\rightarrow\infty}f(x)/g(x) = \infty$;
$f(x) = O(g(x))$ or $f(x) \lesssim g(x)$ means that there exists a constant $C$ such that $f(x)\leq C g(x)$ asymptotically;
$f(x) = \Omega(g(x))$ or $f(x) \gtrsim g(x)$ means that there exists a constant $C$ such that $f(x)\geq C g(x)$ asymptotically;
$f(x) = \Theta(g(x))$ or $f(x) \simeq g(x)$ means that there exists constants $C$ and $C'$ such that $C g(x) \leq f(x)\leq C' g(x)$ asymptotically.

\subsection{Model}
\label{subsec:model}

We now introduce our model assumption. 
Suppose that an unknown matrix $\pmb{M}^*\in\mathbb{R}^{d\times d}$ is symmetric. 
The spectral decomposition of $\pmb{M}^*$ is given by
$$
\pmb{M}^* = \sum_{k\in [d]}\lambda_k(\pmb{M}^*) \pmb{u}_k \pmb{u}_k^\top,
$$
where $\lambda_1(\pmb{M}^*)\geq \cdots \geq \lambda_d(\pmb{M}^*)$ are its eigenvalues and $\pmb{u}_1,\dots,\pmb{u}_d \in \mathbb{R}^d$ are the corresponding eigenvectors.
We assume that $\lambda_1(\pmb{M}^*) > \lambda_2(\pmb{M}^*)$ and the leading eigenvector $\pmb{u}_1$ of $\pmb{M}^*$ is sparse, that is, for some set $J\in [d]$,
$$
\begin{cases}
u_{1,i} \neq 0 & \text{if } i\in J \\
u_{1,i} = 0 & \text{otherwise.}
\end{cases}
$$	
With a notation $supp(\pmb{a}):=\{i\in[d] : a_{i}\neq 0\}$ for any vector $\pmb{a}\in\mathbb{R}^d$, we can write $J = supp(\pmb{u}_1)$.
Also, we denote the size of $J$ by $s$.

\paragraph{Incomplete and noisy observation}

Suppose that we have only noisy observations of the entries of $\pmb{M}^*$ over a sampling set $\Omega \subseteq [d]\times[d]$.
	Specifically, we observe a symmetric matrix $\pmb{M}\in\mathbb{R}^{d\times d}$ such that 
	$$
	M_{i,j}= M_{j,i}= \delta_{i,j}\cdot(M^*_{i,j} + \epsilon_{i,j})
	$$
	for $1\leq i\leq j\leq d$, where $\delta_{i,j} = 1$ if $(i,j)\in\Omega$ and $\delta_{i,j} = 0$ otherwise,
	 and $\epsilon_{i,j}$ is the noise at location $(i,j)$. 
In this paper, we consider the following assumptions on random sampling and random noise: for $1\leq i\leq j\leq d$,
	\begin{itemize}
		\item Each $(i,j)$ is included in the sampling set $\Omega$ independently with probability $p$ (that is, $\delta_{i,j}\overset{i.i.d.}{\sim} Ber(p)$.)
		\item $\delta_{i,j}$'s and $\epsilon_{i,j}$'s are mutually independent.
		\item $\mathbb{E}[\epsilon_{i,j}] = 0$ and $\mathsf{Var}[\epsilon_{i,j}] = \sigma^2$.
		\item $|\epsilon_{i,j}|\leq B$ almost surely.
	\end{itemize}

\section{Main Results}
\label{sec:main_results}

As mentioned in the introduction, we consider the following semidefinite programming (SDP) in order to recover the true support $J$:
\begin{equation}
\label{fps_problem}
\hat{\pmb{X}} = \underset{\pmb{X}\succeq 0 \text{ and } tr(\pmb{X}) = 1}{\arg\max}~ 
\langle \pmb{M}, \pmb{X} \rangle - \rho \|\pmb{X}\|_{1,1},
\end{equation}
where we estimate $J$ by $\hat{J} = supp(diag(\hat{\pmb{X}}))$.
We recall that \eqref{fps_problem} is a convex relaxation of the following nonconvex problem:
\begin{equation}
\label{noncvx_problem}
\hat{\pmb{x}} = \underset{\pmb{x}^\top \pmb{x} = 1}{\arg\max}~ \pmb{x}^\top \pmb{M}\pmb{x} - \rho \|\pmb{x}\|_1^2.
\end{equation}

In Theorem \ref{thm_main},
we will show that under appropriate conditions, the solution of \eqref{fps_problem} attains $\hat{J} = J$ with high probability.
Our main technical tool used in the proof is the primal-dual witness argument \citep{wainwright2009sharp}.
We start with deriving the sufficient conditions for the primal-dual solutions of \eqref{fps_problem} to be uniquely determined and satisfy $supp(diag(\hat{\pmb{X}})) = J$. 
We then establish a proper candidate solution which meets the derived sufficient conditions,
where we make use of the Karush-Kuhn-Tucker (KKT) conditions of \eqref{noncvx_problem} to set up a reasonable candidate.
We finally develop the conditions under which the established candidate solution satisfies the sufficient conditions from the primal-dual witness argument of \eqref{fps_problem} with high probability.
Detailed proof is given in Appendix \ref{sec:proof_of_thm1}.

\begin{theorem}
\label{thm_main}
Under the model defined in Section \ref{subsec:model},
assume that the following conditions hold:
\begin{align*}
&2\sqrt{2}\cdot\frac{K_1 + \rho s}{p(\lambda_1(\pmb{M}^*_{J,J}) - \lambda_2(\pmb{M}^*_{J,J}))}
\leq
\min_{i \in J}|u_{1,i}|,
\\
&\rho
>
2\sqrt{ps^c  \cdot \big\{(1-p) \|\pmb{M}^*_{J^c,J}\|_F^2 + (d-s)s\sigma^2 \big\} } + p\cdot \|\pmb{M}^*_{J^c,J}\|_{\max}
\\
&(K_2 + p\cdot \| \pmb{M}^*_{J^c,J} \|_2)^2\cdot (1+\sqrt{s})^2
\leq
\Big\{ 
p\cdot (\lambda_1(\pmb{M}^*_{J,J}) - \lambda_2(\pmb{M}^*_{J,J}))
-
2\cdot K_1 - 2\rho s \Big\}
\\&~~~~~~~~~~~~~~~~~~~~~~~~~~~~~~~~~~~~~~~~~~~\times
\Big\{
p\cdot (\lambda_1(\pmb{M}^*_{J,J}) - \lambda_1(\pmb{M}^*_{J^c,J^c} ))
- K_1 - K_3 - \rho d\Big\},
\end{align*}
where $c > 0$, and $K_1$, $K_2$ and $K_3$ are defined as follows:
\begin{align*}
K_1 &:= (c+1)\cdot R_1 \log (2s) + \sqrt{2(c+1)}\cdot R_2 \sqrt{\log (2s)} 
\\
K_2 &:= (c+1)\cdot R_3 \log d + \sqrt{2(c+1)}\cdot R_4 \sqrt{\log d}
\\
K_3 &:= (c+1)\cdot R_5 \log(2(d-s)) + \sqrt{2(c+1)}\cdot R_6 \sqrt{\log(2(d-s))}
\end{align*}
and
\begin{align*}
R_1 
&:= \max \{ (1-p) \|\pmb{M}_{J,J}^*\|_{\max} + B,~~ p\|\pmb{M}_{J,J}^*\|_{\max} \},
\\
R_2 
&:= \sqrt{p(1-p)}\|\pmb{M}_{J,J}^*\|_{2,\infty} + \sqrt{ps\sigma^2},
\\
R_3 
&:= \max \{ (1-p) \|\pmb{M}_{J^c,J}^*\|_{\max} + B,~~ p\|\pmb{M}_{J^c,J}^*\|_{\max} \},
\\
R_4 
&:= \max \{ \sqrt{p(1-p)}\|\pmb{M}_{J^c,J}^*\|_{2,\infty} + \sqrt{p(d-s)\sigma^2},
\sqrt{p(1-p)}\|\pmb{M}_{J,J^c}^*\|_{2,\infty} + \sqrt{ps\sigma^2} \},
\\
R_5 
&:= \max \{ (1-p) \|\pmb{M}_{J^c,J^c}^*\|_{\max} + B,~~ p\|\pmb{M}_{J^c,J^c}^*\|_{\max} \},
\\
R_6 
&:= \sqrt{p(1-p)}\|\pmb{M}_{J^c,J^c}^*\|_{2,\infty} + \sqrt{p(d-s)\sigma^2}.
\end{align*}

Then the optimal solution $\hat{\pmb{X}}$ to the problem (\ref{fps_problem}) is unique
and satisfies $supp(diag(\hat{\pmb{X}})) = J$
with probability at least $1-s^{-c} - d^{-c} - (2s)^{-c} - (2(d-s))^{-c}$.
\end{theorem}

To better interpret the conditions of $\pmb{M}^*$ and $p$ listed in Theorem  \ref{thm_main} and understand under what circumstance these conditions hold, we consider the following two particular scenarios:
\begin{itemize}
\item[(s1)] $B$ = $\sigma_2 = 0$, that is, the observation $\pmb{M}$ is noiseless (but still incomplete).
\item[(s2)] The rank of $\pmb{M}^*$ is 1.
\end{itemize}
For both cases, we set $p\geq 0.5$ for simplicity.
Under the first setting, we can re-express the conditions on $\pmb{M}^*$ for exact sparse recovery of $J$ in a more interpretable way (specifically, in terms of coherence parameters and spectral gap) as well as the conditions on $p$.
In the second setting, we aim to investigate that the maximum level of noise that is allowed by Theorem \ref{thm_main}.
Corollaries \ref{cor_1} and \ref{cor_2} include the results of the two settings (s1) and (s2), respectively.

Before elaborating the details, we first define the coherence parameters of the sub-matrices $\pmb{M}^*_{J,J}$, $\pmb{M}^*_{J^c,J}$ and $\pmb{M}^*_{J^c,J^c}$.

\begin{definition}[Coherence parameters]
We define the coherence parameters $\mu_0(\pmb{M}^*_{J,J})$, $\mu_1(\pmb{M}^*_{J,J})$, $\mu_2(\pmb{M}^*_{J^c,J})$ and $\mu_3(\pmb{M}^*_{J^c,J^c})$ as follows:
\begin{align*}
\mu_0(\pmb{M}^*_{J,J}) 
&:= \frac{\|\pmb{M}^*_{J,J}\|_{\max}}{\lambda_1(\pmb{M}^*_{J,J}) - \lambda_2(\pmb{M}^*_{J,J})}
,~~ 
\mu_1(\pmb{M}^*_{J,J})
:= \frac{\|\pmb{M}^*_{J,J}\|_{\max}}{\|\pmb{M}^*_{J,J}\|_{2,\infty}},
\\
\mu_2(\pmb{M}^*_{J^c,J})
&:= \min \bigg\{
\frac{\|\pmb{M}^*_{J^c,J}\|_{\max}}{\|\pmb{M}^*_{J^c,J}\|_{F}},~
\max\Big\{
\frac{\|\pmb{M}^*_{J^c,J}\|_{\max}}{\|\pmb{M}^*_{J^c,J}\|_{2,\infty}},~
\frac{\|\pmb{M}^*_{J^c,J}\|_{\max}}{\|{\pmb{M}^{*\top}_{J^c,J}}\|_{2,\infty}}\Big\},~
\frac{\|\pmb{M}^*_{J^c,J}\|_{\max}}{\|{\pmb{M}^{*\top}_{J^c,J}}\|_{\infty, 2}}
\bigg\},
\\
\mu_3(\pmb{M}^*_{J^c,J^c})
&:= \min \bigg\{
\frac{\|\pmb{M}^*_{J^c,J^c}\|_{\max}}{\|\pmb{M}^*_{J^c,J^c}\|_{2}},~
\frac{\|\pmb{M}^*_{J^c,J^c}\|_{\max}}{\|\pmb{M}^*_{J^c,J^c}\|_{2,\infty}}
\bigg\}.
\end{align*}
We use $\mu_0$, $\mu_1$, $\mu_2$ and $\mu_3$ as shorthand for $\mu_0(\pmb{M}^*_{J,J})$, $\mu_1(\pmb{M}^*_{J,J})$, $\mu_2(\pmb{M}^*_{J^c,J})$ and $\mu_3(\pmb{M}^*_{J^c,J^c})$, respectively.
Intuitively, when each coherence parameter is small, all the entries of the corresponding matrix have comparable magnitudes.
Note that $\frac{1}{s}\leq \mu_0 \leq 1$, 
$\frac{1}{\sqrt{s}}\leq \mu_1 \leq 1$, 
$\frac{1}{\sqrt{s(d-s)}}\leq \mu_2 \leq 1$, 
$\frac{1}{d-s}\leq \mu_3 \leq 1$. 
\end{definition}

\begin{corollary}
\label{cor_1}
Assume that $B$ = $\sigma_2 = 0$, $p\geq 0.5$ and $\min_{i \in J}|u_{1,i}|=\Omega(\frac{1}{\sqrt{s}})$.
Denote $\lambda_1(\pmb{M}^*_{J,J}) - \lambda_2(\pmb{M}^*_{J,J})$ by $\bar{\lambda}(\pmb{M}^*_{J,J})$.
If the following conditions hold:
\begin{align}
\mu_0 
&= o\bigg(\frac{1}{\sqrt{s}\log s}\bigg),
\label{cor1_cond1}
\\
\|\pmb{M}^*_{J^c,J}\|_{\max}
&= o\bigg(\frac{\bar{\lambda}(\pmb{M}^*_{J,J})}{s}
\cdot \min\Big\{\mu_2,~ \frac{1}{s},~ \frac{\sqrt{s}}{\log d} \Big\}
\bigg),
\label{cor1_cond2}
\\
\|\pmb{M}^*_{J^c,J^c}\|_{\max}
&= o\bigg(\bar{\lambda}(\pmb{M}^*_{J,J})
\cdot \min\Big\{\mu_3,~ \frac{1}{\log (d-s)} \Big\}
\bigg),
\label{cor1_cond3}
\\
\sqrt{\frac{1-p}{p}}
&= o\bigg(\min \Big\{
\mu_1\sqrt{\log s},
\nonumber
\\
&~~~~~~~~~~~~~~
\frac{\bar{\lambda}(\pmb{M}^*_{J,J}) \mu_2}{\|\pmb{M}^*_{J^c,J}\|_{\max}}\cdot
\min \Big\{\frac{1}{s^2\sqrt{s}},~ \frac{1}{s\sqrt{s(d-s)}} \Big\},
\label{cor1_cond4}
\\
&~~~~~~~~~~~~~~
\frac{\bar{\lambda}(\pmb{M}^*_{J,J}) \mu_3}{\|\pmb{M}^*_{J^c,J^c}\|_{\max}}\cdot \frac{1}{\sqrt{\log(d-s)}}
\Big\}
\bigg),
\nonumber
\\
\rho
&= \Theta\bigg(\frac{p\bar{\lambda}(\pmb{M}^*_{J,J})}{s^2}
\bigg)
\label{cor1_cond5},
\end{align}
then 
the conditions in Theorem \ref{thm_main} hold asymptotically, that is,
when $s$ and $d$ are sufficiently large,
the optimal solution $\hat{\pmb{X}}$ to the problem (\ref{fps_problem}) is unique
and satisfies $supp(diag(\hat{\pmb{X}})) = J$
with probability at least $1-s^{-1} - d^{-1} - (2s)^{-1} - (2(d-s))^{-1}$.
\end{corollary}

\paragraph{Conditions on true matrix $\pmb{M}^*$}
From the conditions in Corollary \ref{cor_1}, we can find desirable properties on the matrix $\pmb{M}^*$ as follows:

\begin{itemize}
\item {\it Incoherence of $\pmb{M}^*_{J,J}$, and coherence of $\pmb{M}^*_{J^c,J}$ and $\pmb{M}^*_{J^c,J^c}$}: 
From the coherence parameter in \eqref{cor1_cond1} and those in \eqref{cor1_cond2}, \eqref{cor1_cond3} and \eqref{cor1_cond4}, we see that the sub-matrix $\pmb{M}^*_{J,J}$ and the sub-matrices $\pmb{M}^*_{J^c,J}$ and $\pmb{M}^*_{J^c,J^c}$ are expected to be incoherent and coherent, respectively.
This is different from other problems involving incomplete matrices, such as matrix completion \citep{candes2009exact} and standard PCA on incomplete data \citep{cai2021subspace},
where the entire matrix, not a sub-matrix, is required to be incoherent.

We can easily check the need of incoherence of $\pmb{M}^*_{J^c,J}$ with an example that the sub-matrix has only one entry with a large magnitude while the other entries have relatively small values.
Even if the true leading eigenvector of the sub-matrix is not sparse, the sparse PCA algorithm may produce a solution $\hat{J}$ which has a smaller size than that of the true support $J$.

However, for $\pmb{M}^*_{J^c,J}$ and $\pmb{M}^*_{J^c,J^c}$, coherence is preferable: intuitively speaking,
when $\pmb{M}^*_{J^c,J}$ and $\pmb{M}^*_{J^c,J^c}$ are the most coherent, that is, only one entry is nonzero in each sub-matrix, and all other entries are zero,
missing the entries in $\pmb{M}^*_{J^c,J}$ and $\pmb{M}^*_{J^c,J^c}$ does not change the leading eigenvector of $\pmb{M}^*$.
On the other hand, when $\pmb{M}^*_{J^c,J}$ and $\pmb{M}^*_{J^c,J^c}$ are incoherent, that is, all the entries have comparable magnitudes,
missing only a few entries changes the leading eigenvector and its sparsitency, so that sparse PCA is likely to fail to recover $J$. A simple illustration can be found in the Appendix \ref{sec:example_coherence}.

\item {\it Large spectral gap $\bar{\lambda}(\pmb{M}^*_{J,J})$ ($= \lambda_1(\pmb{M}^*_{J,J}) - \lambda_2(\pmb{M}^*_{J,J})$)}:
This can be found in \eqref{cor1_cond2}, \eqref{cor1_cond3} and \eqref{cor1_cond4}.
A sufficiently large spectral gap requirement has been also discussed in the work on sparse PCA on the complete matrix \citep{lei2015sparsistency}.
It ensures the uniqueness and identifiability of the orthogonal projection matrix with respect to the principal subspace.
If the spectral gap of eigenvalues is nearly zero, then the top two eigenvectors are indistinguishable given the observational noise, leading to failure to recover the sparsity of the leading eigenvector.
\\
We also note that $\lambda_1(\pmb{M}^*_{J,J}) - \lambda_2(\pmb{M}^*_{J,J}) \geq \lambda_1(\pmb{M}^*) - \lambda_2(\pmb{M}^*)$ since $\lambda_1(\pmb{M}^*_{J,J}) = \lambda_1(\pmb{M}^*)$ and $\lambda_2(\pmb{M}^*_{J,J}) \leq \lambda_2(\pmb{M}^*)$.
Hence, a large $\lambda_1(\pmb{M}^*) - \lambda_2(\pmb{M}^*)$ implies a large $\bar{\lambda}(\pmb{M}^*_{J,J})$.
\item {\it Small magnitudes of $\pmb{M}^*_{J^c,J}$ and $\pmb{M}^*_{J^c,J^c}$}:
This can also be found in \eqref{cor1_cond2}, \eqref{cor1_cond3} and \eqref{cor1_cond4}.
This condition is also natural: if the magnitudes are relatively small, missing the entries will not make a big impact to the result.
\end{itemize}

\paragraph{Conditions on $p$ (ratio of missing data)}
For simplicity, suppose that $\bar{\lambda}(\pmb{M}^*_{J,J}) = O(s)$ and $s=O(\log d)$.
Then from the conditions \eqref{cor1_cond2} and \eqref{cor1_cond3}, we can write
$\|\pmb{M}^*_{J^c,J}\|_{\max}
= \epsilon_1\cdot \min\Big\{\mu_2,~ \frac{1}{s} \Big\}
$ for some $\epsilon_1 = o(1)$
and 
$\|\pmb{M}^*_{J^c,J^c}\|_{\max}
= \epsilon_2\cdot \min\Big\{s\mu_3,~ \frac{s}{\log d} \Big\}$
 for some $\epsilon_2 = o(1)$.

With these notations, we can write the condition \eqref{cor1_cond4} as follows:
$$
\sqrt{\frac{1-p}{p}}
= o\bigg(\min \Big\{
\mu_1\sqrt{\log s},~
\frac{\mu_2}{\epsilon_1}\cdot
\frac{\frac{1}{\sqrt{sd}}}{\min\Big\{\mu_2,~ \frac{1}{s}\Big\}},~
\frac{\mu_3}{\epsilon_2}\cdot \frac{1}{\min\Big\{\mu_3\sqrt{\log d},~ \frac{1}{\sqrt{\log d}} \Big\}}
\Big\}
\bigg).
$$
From the above equation, we can see that the matrix coherence ($\mu_1, \mu_2, \mu_3$) and the matrix magnitudes (in terms of $\epsilon_1$ and $\epsilon_2$) affect the expected number of entries to be observed, as well as $d$ and $s$.
Let us consider two extreme cases where the coherence parameters are maximized and minimized.
We discuss the bound of the sample complexity in each case.

\begin{itemize}
\item {\it The best scenario where the bound of the sample complexity is the lowest}:
Suppose that $\mu_1 = o(\frac{1}{\log s})$ and
$\mu_2 = \mu_3 = 1$ (note that when $\mu_0 = o\big(\frac{1}{\sqrt{s}\log s}\big)$, $\mu_1$ is upper bounded by $o\big(\frac{1}{\log s}\big)$.)
Then the condition \eqref{cor1_cond4} can be written as:
$$
\sqrt{\frac{1-p}{p}}
= o\bigg(\min \Big\{
\frac{1}{\sqrt{\log s}},~~
\frac{1}{\epsilon_1}\cdot \sqrt{\frac{s}{d}},~~
\frac{\sqrt{\log d}}{\epsilon_2}
\Big\}
\bigg)
= o\bigg(\min \Big\{
\frac{1}{\sqrt{\log s}},~~
\frac{1}{\epsilon_1}\cdot \sqrt{\frac{s}{d}}
\Big\}
\bigg).
$$
As $\epsilon_1$ is smaller (i.e., the magnitudes of the entries of $\pmb{M}^*_{J^c,J}$ are smaller,) the bound of $p$ is allowed to be smaller.
In the best case, $\sqrt{\frac{1-p}{p}} = o((\log s)^{-0.5})$,
that is,
$p = \omega\Big(\frac{1}{(\log s)^{-1} + 1}\Big)$.
\item {\it The worst scenario where the bound of the sample complexity is the highest}: 
Suppose that 
$\mu_1 = \frac{1}{\sqrt{s}}$, $\mu_2 = \frac{1}{\sqrt{s(d-s)}}$ and $\mu_3 = \frac{1}{d-s}$.
In this case, the condition \eqref{cor1_cond4} can be written as:
$$
\sqrt{\frac{1-p}{p}}
= o\bigg(\min \Big\{
\sqrt{\frac{\log s}{s}},~
\frac{1}{\epsilon_1}\cdot \frac{1}{\sqrt{sd}},~
\frac{1}{\epsilon_2}\cdot \frac{1}{\sqrt{\log d}}
\Big\}
\bigg).
$$
Suppose that $\epsilon_1$ and $\epsilon_2$ are not as small as $\frac{1}{\sqrt{s}}$. Then $\sqrt{\frac{1-p}{p}}$ is at most $o(d^{-0.5})$,
that is, $p = \omega\Big(\frac{1}{d^{-1} + 1}\Big)$.
\end{itemize}

Next, we consider the second setting (s2) where the rank of $\pmb{M}^*$ is assumed to be $1$, that is, 
$\pmb{M}^* = \lambda_1(\pmb{M}^*) \pmb{u}_1 \pmb{u}_1^\top$ (without loss of generality, we assume
$\lambda_1(\pmb{M}^*) > 0$.)
Trivially, $\pmb{M}^*_{J^c, J} = \pmb{M}^*_{J, J^c} = \pmb{M}^*_{J^c, J^c} = 0$ and Theorem \ref{thm_main} can be greatly simplified.
Here, we focus on analyzing how much noise (parameters $B$ and $\sigma^2$) is allowed.

\begin{corollary}
\label{cor_2}
Assume that $p\geq 0.5$ and the rank of $\pmb{M}^*$ is $1$, that is, 
$\pmb{M}^* = \lambda_1(\pmb{M}^*) \pmb{u}_1 \pmb{u}_1^\top$. Let $\lambda_1(\pmb{M}^*) > 0$.
Suppose that $s$ and $d$ satisfy
$
\frac{1}{\sqrt{s}}
\leq
\frac
{12 + \frac{d-s}{s} + 8\sqrt{2}a_2 - \sqrt{ (4-\frac{d-s}{s} - 8\sqrt{2}a_2)^2 + 512 a_1^2 (1+\sqrt{s})^2 }}
{4\sqrt{2} + \sqrt{2}\cdot \frac{d-s}{s} + 16a_2 - 16\sqrt{2}a_1^2(1+\sqrt{s})^2}
$
where
$
a_1 
=
(2-\frac{1}{p})\cdot \frac{\log d}{8\sqrt{2}\log(2s)}
+
\frac{\sqrt{\max\{d-s, s\}}\cdot \sqrt{\log d}}{16 s^2\sqrt{d-s}}$
and
$
a_2
=
(2-\frac{1}{p})\cdot \frac{\log(2(d-s))}{8\sqrt{2}\log(2s)}
+
\frac{\sqrt{\log(2(d-s))}}{16 s^2}$.
If the following conditions hold:
\begin{align*}
\frac{\max_{i,j\in J} |u_{1,i} u_{1,j}|}{\min_{i \in J}|u_{1,i}|}
&\leq
\frac{1}{16\sqrt{2}\log(2s)},
\\
\frac{\max_{i\in J} |u_{1,i}|}{\min_{i \in J}|u_{1,i}|}
&\leq
\frac{1}{16\sqrt{2}\sqrt{\log(2s)}}\cdot \sqrt{\frac{p}{1-p}},
\\
B 
&\leq 
(2p-1)\lambda_1(\pmb{M}^*)\cdot\max_{i,j\in J} |u_{1,i} u_{1,j}|,
\\
2\sqrt{2}\cdot \sqrt{p\sigma^2 s^2 (d-s)}
<
\rho
&\leq
\frac{1}{8\sqrt{2}s}\cdot p \lambda_1(\pmb{M}^*)\cdot \min_{i \in J}|u_{1,i}|,
\end{align*}
then the optimal solution $\hat{\pmb{X}}$ to the problem (\ref{fps_problem}) is unique
and satisfies $supp(diag(\hat{\pmb{X}})) = J$
with probability at least $1-s^{-1} - d^{-1} - (2s)^{-1} - (2(d-s))^{-1}$.
\end{corollary}

\paragraph{Conditions on noise parameters $B$ and $\sigma^2$}
For simplicity, let ${\lambda_1}(\pmb{M}^*) = O(s)$ and $\forall |u_{1,i}| = \Theta(\frac{1}{\sqrt{s}})$.
Then the above conditions in Corollary \ref{cor_2} imply that
$$
B \lesssim p ~~\text{~and~}~~ \sigma^2 \lesssim \frac{p}{s^3(d-s)}.
$$
The condition for $B$ is relatively moderate while $\sigma^2$ needs to be extremely small to satisfy the condition in Corollary \ref{cor_2}. 
We comment this is only a sufficient condition, and the experimental results show that \eqref{fps_problem} can succeed even with $\sigma^2$ larger than the aforementioned bound.

\section{Numerical Results}
\label{sec:numerical_results}

We perform the SDP algorithm of \eqref{fps_problem} on synthetic and real data to 
validate our theoretic results and
show how well the true support of the sparse principal component is exactly recovered.
Our experiments were executed on {\tt MATLAB} and standard {\tt CVX} code was used, although very efficient scalable SDP solvers exist in practice \citep{yurtsever2021scalable}.

\subsection{Synthetic Data}

We perform two lines of experiments:
\begin{enumerate}
\item With the spectral gap $\lambda_1(\pmb{M}^*)-\lambda_2(\pmb{M}^*)$ and the noise parameters $B$ and $\sigma^2$ fixed, we compare the results for different $s$ and $d$.
\item With $s$ and $d$ fixed, we compare the results for different spectral gaps and noise parameters.
\end{enumerate}

In each experiment, we generate the true matrix $\pmb{M}^*$ as follows: 
the leading eigenvector $\pmb{u}_1$ is set to have $s$ number of non-zero entries.
$\lambda_2(\pmb{M}^*), \dots, \lambda_d(\pmb{M}^*)$ are randomly selected from a normal distribution with mean $0$ and standard deviation $1$,
and $\lambda_1(\pmb{M}^*)$ is set to $\lambda_2(\pmb{M}^*)$ plus the spectral gap.
The orthogonal eigenvectors are randomly selected, while the non-zero entries of the leading eigenvector $\pmb{u}_1$ are made to have a value of at least $\frac{1}{2\sqrt{s}}$.

When generating the observation $\pmb{M}$, we first add to $\pmb{M}^*$ the entry-wise noise which is randomly selected from a truncated normal distribution with support $[-B, B]$.
The normal distribution to be truncated is set to have mean $0$ and standard deviation $\sigma_{normal}$.
After adding the entry-wise noise, we generate an incomplete matrix $\pmb{M}$ by selecting the observed entries uniformly at random with probability $p\in \{0.1, 0.3, 0.5, 0.7, 0.9\}$.

In each setting, we run the algorithm \eqref{fps_problem} and verify if the solution exactly recovers the true support.
We repeat each experiment $30$ times with different random seeds, and calculate the rate of exact recovery in each setting.

\paragraph{Experiment 1}

In this experiment, we fix the spectral gap $\lambda_1(\pmb{M}^*)-\lambda_2(\pmb{M}^*)$ as $20$ and the noise parameters $B$ and $\sigma^2$ as $5$ and $0.01$.
We use the tuning parameter $\rho = 0.1$.
We try three different matrix dimensions $d\in\{20, 50, 100\}$ and three different support sizes $s\in\{5, 10, 20\}$.

To check whether the bound of the sample complexity obtained in Corollary \ref{cor_1} is tight,
we calculate the coherence parameters and the maximum magnitudes of the sub-matrices at each setting,
and calculate the following rescaled parameter:
$$
\sqrt{\frac{p}{1-p}}
\cdot \min \Big\{
\mu_1\sqrt{\log s},
\frac{\bar{\lambda}(\pmb{M}^*_{J,J}) \mu_2}{\|\pmb{M}^*_{J^c,J}\|_{\max}}\cdot
\min \Big\{\frac{1}{s^2\sqrt{s}},~ \frac{1}{s\sqrt{s(d-s)}} \Big\},
\frac{\bar{\lambda}(\pmb{M}^*_{J,J}) \mu_3}{\|\pmb{M}^*_{J^c,J^c}\|_{\max}}\cdot \frac{1}{\sqrt{\log(d-s)}}
\Big\},
$$
which is derived from \eqref{cor1_cond4}.
If the exact recovery rate versus this rescaled parameter is the same across different settings, then we empirically justify that the bound of the sample complexity we derive is "tight" in the sense that the exact recovery rate is solely determined by this rescaled parameter.

Figure \ref{fig:simulation1} shows the experimental results.
The two plots above are the experimental results for different values of $s$ when $d=100$,
and the two plots below are for different values of $d$ when $s=10$.
The x-axis of the left graphs represents $p$, and the x-axis of the right graphs indicates the rescaled parameter.

We can see from the two graphs on the right that the exact recovery rate versus the rescaled parameter is the same in different settings of $d$ and $s$.
This means that our bound of the sample complexity is tight.

Another observation we can make is that the exact recovery rate is not necessarily increasing or decreasing as $s$ or $d$ increases or decreases.
This is probably because coherences and maximum magnitudes of sub-matrices are involved in the sample complexity as well.

\begin{figure}
  \centering
  \includegraphics[width=0.5\textwidth]{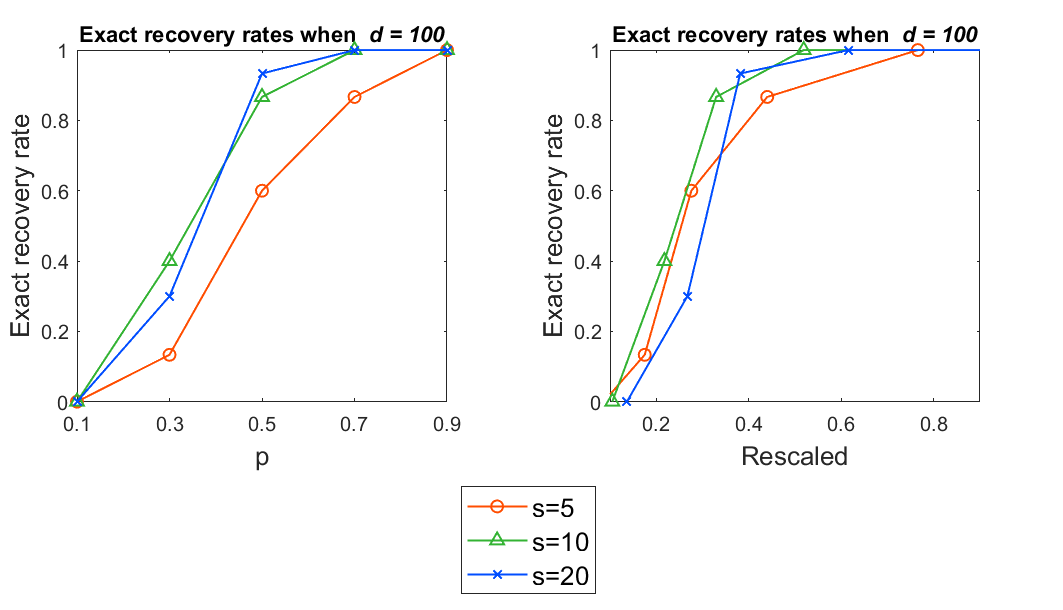}
  \includegraphics[width=0.5\textwidth]{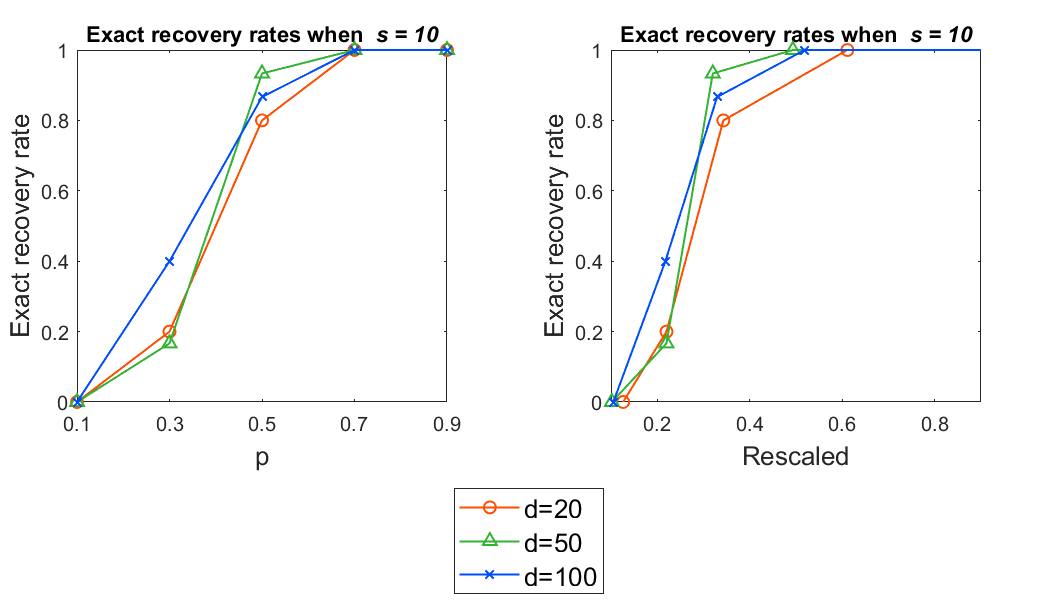}
  \caption{Results of experiment 1 on synthetic data.}
  \label{fig:simulation1}
\end{figure}

\paragraph{Experiment 2}

Here, we fix the matrix dimension $d$ as $100$ and the support size $s$ as $50$.
We set $B=5$.
We try three different spectral gaps $\lambda_1(\pmb{M}^*)-\lambda_2(\pmb{M}^*)\in\{10, 30, 50\}$ and three different standard deviations of the normal distribution, $\sigma_{normal}\in \{0.1, 0.3, 0.5\}$.
We try two different tuning parameters $\rho \in \{0.1, 0.01\}$ and report the best result. 

Figure \ref{fig:simulation2} demonstrates the experimental results.
The three plots show the results when $\sigma_{normal}$ is $0.1$, $0.3$ and $0.5$, respectively.
The red, green and blue lines indicate the cases where the spectral gap $\lambda_1(\pmb{M}^*)-\lambda_2(\pmb{M}^*)$ is $50$, $30$ and $10$, respectively.
From the plots, we can observe that the exact recovery rate increases as $\sigma^2$ is small and $\lambda_1(\pmb{M}^*)-\lambda_2(\pmb{M}^*)$ is large,
which is consistent with the conditions we have checked in Corollaries \ref{cor_1} and \ref{cor_2}.

\begin{figure}
  \centering
  \includegraphics[width=0.7\textwidth]{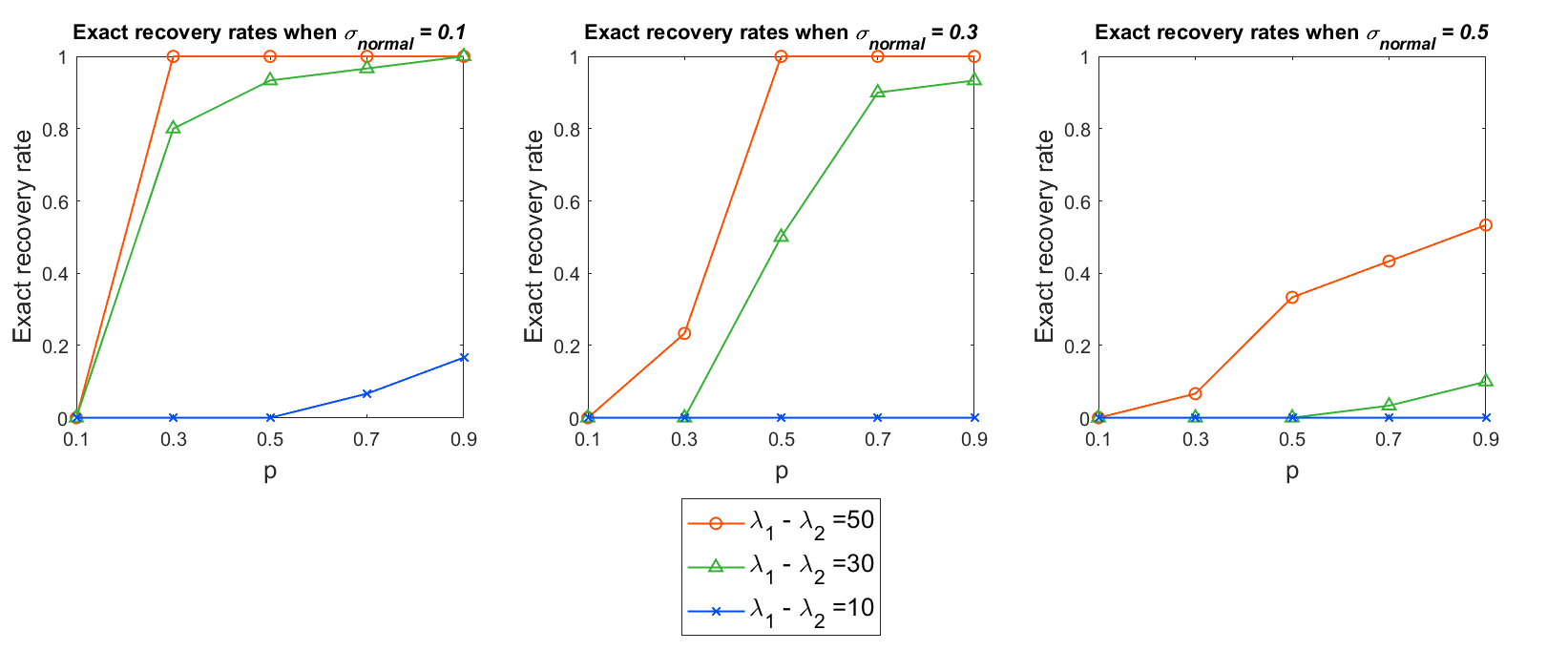}
  \caption{Results of experiment 2 on synthetic data.}
  \label{fig:simulation2}
\end{figure}

\subsection{Gene Expression Data}

We analyze a gene expression dataset (GSE21385) from the Gene Expression Omnibus website (\url{https://www.ncbi.nlm.nih.gov/geo/}.)
The dataset examines rheumatoid arthritis synovial fibroblasts, which together with synovial macrophages, are the two leading cell types that invade and degrade cartilage and bone.

The original data set contains $56$ subjects and $112$ genes.
We compute its incomplete covariance matrix,
where $87$\% of the matrix entries are observed since some subject/gene pairs are unobserved.
With this incomplete covariance matrix, we solve the semidefinite program in \eqref{fps_problem} for sparse PCA with $\rho=2$.

By solving \eqref{fps_problem}, we find that the support of the solution contains $3$ genes: beta-1 catenin (CTNNB), hypoxanthine-guanine phosphoribosyltransferase 1 (HPRT1) and semaphorin III/F (SEMA3F).
Our result is consistent with prior studies on rheumatoid arthritis since CTNNB has been found to be upregulated \citep{iwamoto2018osteogenic}, SEMA3F has been found to be downregulated \citep{tang2018class}, and HPRT1 is known to be a housekeeping gene \citep{mesko2013peripheral}.

\section{Concluding Remarks}
\label{sec:concluding_remarks}

We have presented the sufficient conditions to exactly recover the true support of the sparse leading eigenvector by solving a simple semidefinite programming on an incomplete and noisy observation.
We have shown that the conditions involve matrix coherence, spectral gap, matrix magnitudes, sample complexity and variance of noise, and provided empirical evidence to justify our theoretical results.
To the best of our knowledge, we provide the first theoretical guarantee for exact support recovery with sparse PCA on incomplete data.
While we currently focus on a uniformly missing at random setup, an interesting open question is whether it is possible to provide guarantees for a deterministic pattern of missing entries.

\bibliography{ref}
\bibliographystyle{plainnat}

\newpage

\appendix

\section{Examples of coherent and incoherent sub-matrices $\pmb{M}^*_{J^c, J}$ and $\pmb{M}^*_{J^c, J^c}$}
\label{sec:example_coherence}

For illustration, here we present a noiseless case (that is, we only focus on the change of eigen-structure caused by missing values,) and set $J = \{1,\dots, s\}$, i.e., only the first $s$ entries of the true leading eigenvector are nonzero.
We let $d=20$ and $s=10$ in the below examples.

In the following four examples, 
we show a complete or incomplete matrix, followed by its leading eigenvector.
We separate each matrix and its leading eigenvector by an arrow.
The entries in $J^c\times J$, $J\times J^c$ and $J^c\times J^c$ sub-matrices are marked in bold.
The missing entries are marked in red.

\begin{center}
$\left(\begin{smallmatrix}
1 & \cdots & 1 & 1 & 1 & \pmb{0} & \pmb{\color{red} 0} & \cdots & \pmb{0} \\
\vdots & \vdots & \vdots & \vdots & \vdots & \vdots & \vdots & \vdots & \vdots \\
1 & \cdots & 1 & 1 & 1 & \pmb{\color{red} 0} & \pmb{0} & \cdots & \pmb{0} \\
1 & \cdots & 1 & 2 & 0 & \pmb{1} & \pmb{0} & \cdots & \pmb{\color{red} 0} \\
1 & \cdots & 1 & 0 & 2 & \pmb{-1} & \pmb{0} & \cdots & \pmb{0} \\
\pmb{0} & \cdots & \pmb{\color{red} 0} & \pmb{1} & \pmb{-1} & \pmb{\color{red} 1} & \pmb{0} & \cdots & \pmb{\color{red} 0} \\
\pmb{\color{red} 0} & \cdots & \pmb{0} & \pmb{0} & \pmb{0} & \pmb{0} & \pmb{0} & \cdots & \pmb{0} \\
\vdots & \vdots & \vdots & \vdots & \vdots & \vdots & \vdots & \vdots & \vdots \\
\pmb{0} & \cdots & \pmb{0} & \pmb{\color{red} 0} & \pmb{0} & \pmb{\color{red} 0} & \pmb{0} & \cdots & \pmb{0}
\end{smallmatrix}\right)
\Rightarrow
\left(\begin{smallmatrix}
0.3162 \\
\vdots \\
0.3162 \\
0 \\
\vdots \\
0
\end{smallmatrix}\right)$,~~
$\left(\begin{smallmatrix}
1 & \cdots & 1 & 1 & 1 & \pmb{0} & \pmb{\color{red} 0} & \cdots & \pmb{0} \\
\vdots & \vdots & \vdots & \vdots & \vdots & \vdots & \vdots & \vdots & \vdots \\
1 & \cdots & 1 & 1 & 1 & \pmb{\color{red} 0} & \pmb{0} & \cdots & \pmb{0} \\
1 & \cdots & 1 & 2 & 0 & \pmb{1} & \pmb{0} & \cdots & \pmb{\color{red} 0} \\
1 & \cdots & 1 & 0 & 2 & \pmb{-1} & \pmb{0} & \cdots & \pmb{0} \\
\pmb{0} & \cdots & \pmb{\color{red} 0} & \pmb{1} & \pmb{-1} & \pmb{\color{red} 0} & \pmb{0} & \cdots & \pmb{\color{red} 0} \\
\pmb{\color{red} 0} & \cdots & \pmb{0} & \pmb{0} & \pmb{0} & \pmb{0} & \pmb{0} & \cdots & \pmb{0} \\
\vdots & \vdots & \vdots & \vdots & \vdots & \vdots & \vdots & \vdots & \vdots \\
\pmb{0} & \cdots & \pmb{0} & \pmb{\color{red} 0} & \pmb{0} & \pmb{\color{red} 0} & \pmb{0} & \cdots & \pmb{0}
\end{smallmatrix}\right)
\Rightarrow
\left(\begin{smallmatrix}
0.3162 \\
\vdots \\
0.3162 \\
0 \\
\vdots \\
0
\end{smallmatrix}\right)$
\end{center}

The example on the left is a complete matrix having coherent sub-matrices $\pmb{M}^*_{J^c, J}$ and $\pmb{M}^*_{J^c, J^c}$, 
and the example on the right is its incomplete counterpart.
We can observe that missing some entries does not change the leading eigenvector in this case.

\begin{center}
$\left(\begin{smallmatrix}
  4       & 2      & \cdots & 4      & 2      &  \pmb{1}     & \pmb{-1}     & \cdots  & \pmb{-1}    \\
 2       & 4      & \cdots & 2      & 4      & \pmb{-1}     &  \pmb{1}     & \cdots  &  \pmb{1}    \\
 \vdots  & \vdots & \vdots & \vdots & \vdots & \vdots & \vdots & \vdots  & \vdots \\ 
 4       & 2      & \cdots & 4      & 2      &  \pmb{1}     & \pmb{-1}     & \cdots  & \pmb{-1}    \\
 2       & 4      & \cdots & 2      & 4      & \pmb{\color{red} -1}     &  \pmb{1}     & \cdots   &  \pmb{1}    \\
 \pmb{1}      & \pmb{-1}     & \cdots &  \pmb{1}     & \pmb{\color{red} -1}      & \pmb{1}      & \pmb{-1}     & \cdots   & \pmb{-1}    \\
 \pmb{-1}     &  \pmb{1}     & \cdots & \pmb{-1}     &  \pmb{1}      & \pmb{-1}     &  \pmb{1}     & \cdots  &  \pmb{1}    \\
 \vdots & \vdots & \vdots & \vdots & \vdots  & \vdots & \vdots & \vdots  & \vdots  \\
 \pmb{-1}     &  \pmb{1}     & \cdots & \pmb{-1}     &  \pmb{1}      & \pmb{-1}     &  \pmb{1}     & \cdots  &  \pmb{1}  
\end{smallmatrix}\right)
\Rightarrow
\left(\begin{smallmatrix}
0.3162 \\
\vdots \\
0.3162 \\
0 \\
\vdots \\
0
\end{smallmatrix}\right)$,~~
$\left(\begin{smallmatrix}
  4       & 2      & \cdots & 4      & 2      &  \pmb{1}     & \pmb{-1}     & \cdots  & \pmb{-1}    \\
 2       & 4      & \cdots & 2      & 4      & \pmb{-1}     &  \pmb{1}     & \cdots  &  \pmb{1}    \\
 \vdots  & \vdots & \vdots & \vdots & \vdots & \vdots & \vdots & \vdots  & \vdots \\ 
 4       & 2      & \cdots & 4      & 2      &  \pmb{1}     & \pmb{-1}     & \cdots  & \pmb{-1}    \\
 2       & 4      & \cdots & 2      & 4      & \pmb{\color{red} 0}     &  \pmb{1}     & \cdots   &  \pmb{1}    \\
 \pmb{1}      & \pmb{-1}     & \cdots &  \pmb{1}     & \pmb{\color{red} 0}      & \pmb{1}      & \pmb{-1}     & \cdots   & \pmb{-1}    \\
 \pmb{-1}     &  \pmb{1}     & \cdots & \pmb{-1}     &  \pmb{1}      & \pmb{-1}     &  \pmb{1}     & \cdots  &  \pmb{1}    \\
 \vdots & \vdots & \vdots & \vdots & \vdots  & \vdots & \vdots & \vdots  & \vdots  \\
 \pmb{-1}     &  \pmb{1}     & \cdots & \pmb{-1}     &  \pmb{1}      & \pmb{-1}     &  \pmb{1}     & \cdots  &  \pmb{1}  
\end{smallmatrix}\right)
\Rightarrow
\left(\begin{smallmatrix}
0.3172 \\
\vdots \\
0.3155 \\
0.0115 \\
-0.001 \\
\vdots \\
-0.001
\end{smallmatrix}\right)$
\end{center}

However, as shown above, if the sub-matrices $\pmb{M}^*_{J^c, J}$ and $\pmb{M}^*_{J^c, J^c}$ are highly incoherent, then missing only one entry in the sub-matrix $\pmb{M}^*_{J^c, J}$ changes the leading eigenvector significantly.
In this case, the support of the leading eigenvector of the incomplete matrix is $\{1,\dots,d\}$, so that it is more difficult to exactly recover the true support $J$.

\section{Proof of Theorem \ref{thm_main}}
\label{sec:proof_of_thm1}

We use the primal-dual witness construction \citep{wainwright2009sharp} to obtain the sufficient conditions for support recovery. 
The following proposition indicates the sufficient conditions to recover the support without false positives by using the optimization problem (\ref{fps_problem}).

\begin{proposition}
\label{prop_fps_kkt}
If $\hat{\pmb{X}} \in \mathbb{R}^{d\times d}$, $\hat{\pmb{Z}} \in \mathbb{R}^{d\times d}$ and $\hat\mu \in \mathbb{R}$
satisfy the following:
\begin{align*}
&\hat{\pmb{X}}_{J,J} \succeq 0, tr(\hat{\pmb{X}}_{J,J}) = 1
\\&
\pmb{M}_{J,J} - \rho  \hat{\pmb{Z}}_{J,J} \preceq \hat\mu \pmb{I}
\\&
\pmb{M} - \rho  \hat{\pmb{Z}} \preceq \hat\mu \pmb{I}
\\&
\hat{Z}_{ij} \in \partial |\hat{X}_{ij}| ~~\text{ for each } (i,j) \in J\times J
\\&
\hat{Z}_{ij} \in (-1, 1) ~~\text{ for each } (i,j) \notin J\times J
\\&
(\pmb{M}_{J,J} - \rho  \hat{\pmb{Z}}_{J,J})\hat{\pmb{X}}_{J,J} = \hat\mu \cdot \hat{\pmb{X}}_{J,J}
\\&
(\pmb{M}_{J^c,J} - \rho  \hat{\pmb{Z}}_{J^c,J})\hat{\pmb{X}}_{J,J} = 0,
\end{align*}
then $\hat{\pmb{X}}$ is an optimal solution to the problem (\ref{fps_problem}),
and it satisfies $supp(diag(\hat{\pmb{X}})) \subseteq J$.
\end{proposition}

The proof is provided in \ref{subsec:pf_prop_fps_kkt}.
Recall that our goal is to exactly recover the support $J$ where $s=|J|$,
and the conditions in Proposition \ref{prop_fps_kkt} do not guarantee $supp(diag(\hat{\pmb{X}})) = J$ yet.
We now want to construct a solution $\hat{\pmb{X}}$ satisfying not only the above optimality conditions but also $supp(diag(\hat{\pmb{X}})) = J$.
To find a reasonable candidate, we look at the KKT conditions of the nonconvex problem (\ref{noncvx_problem}) and apply the primal-dual witness argument.
We note that the problem (\ref{noncvx_problem}) is only used for
getting some initial intuition in order to later construct a desirable solution to the problem (\ref{fps_problem}), and solving the problem (\ref{noncvx_problem}) is not our interest.
Proposition \ref{prop_kkt} represents the sufficient conditions for the desirable solution to be uniquely obtained.
The proof is deferred to \ref{subsec:pf_prop_kkt}.

\begin{proposition}
\label{prop_kkt}
Consider a 3-tuple of $(\hat{\pmb{x}}, \hat{\pmb{z}}, \hat{\pmb{w}}) \in \mathbb{R}^{s}\times \mathbb{R}^{s} \times \mathbb{R}^{d-s}$ such that
\begin{align*}
&\hat{z}_i = \text{sign}(u_{1,i}) \text{~~~~for all~} i\in J,
\\
& \hat{\pmb{x}} \text{~is the leading eigenvector of~} \pmb{M}_{J,J}-\rho \hat{\pmb{z}} \hat{\pmb{z}}^\top,
\\
& \hat{\pmb{w}} = \frac{1}{\rho \|\hat{\pmb{x}}\|_1} \pmb{M}_{J^c, J}\hat{\pmb{x}}.
\end{align*}
If the 3-tuple $(\hat{\pmb{x}}, \hat{\pmb{z}}, \hat{\pmb{w}}) \in \mathbb{R}^{s}\times \mathbb{R}^{s} \times \mathbb{R}^{d-s}$ satisfies the following conditions:
\begin{align}
& \text{sign}(\hat{x}_{i}) = \text{sign}(u_{1,i}) \text{~~~~for all~} i\in J
\label{cond1}
\\ 
& \|\hat{\pmb{w}}\|_\infty < 1
\label{cond2}
\\
&
\lambda_1(\pmb{M}_{J,J}-\rho \hat{\pmb{z}} \hat{\pmb{z}}^\top) = \lambda_1(\pmb{M}-\rho (\hat{\pmb{z}}^\top, \hat{\pmb{w}}^\top)^\top (\hat{\pmb{z}}^\top, \hat{\pmb{w}}^\top))
\label{cond3}
\\&
\lambda_1(\pmb{M}_{J,J}-\rho \hat{\pmb{z}} \hat{\pmb{z}}^\top) 
> \lambda_2(\pmb{M}_{J,J}-\rho \hat{\pmb{z}} \hat{\pmb{z}}^\top)
\label{cond4}
\end{align}
then for $\hat{\pmb{X}} := \begin{pmatrix}
\hat{\pmb{x}} \hat{\pmb{x}}^\top & 0 \\
0 & 0
\end{pmatrix}$, 
$\hat{\pmb{X}}$ 
is a unique optimal solution to the problem (\ref{fps_problem}) and satisfies $supp(diag(\hat{\pmb{X}})) = J$.
\end{proposition}

For clarity of exposition, our abuse of notation seemingly assumes $J=[s]$ when we join vectors and matrices, for instance when we join $(\hat{\pmb{z}}^\top, \hat{\pmb{w}}^\top)$ and $\left(\begin{smallmatrix}
\hat{\pmb{x}} \hat{\pmb{x}}^\top & 0 \\
0 & 0
\end{smallmatrix}\right)$.
It should be clear that for $J\neq [s]$, one will need to properly interleave vector entries or matrix rows/columns.

What remains is to derive the sufficient conditions for (\ref{cond1}) - (\ref{cond4}).
Lemmas \ref{lemma_suff_cond_sign} to \ref{lemma_suff_cond_strict_eig} below
presents the sufficient conditions for (\ref{cond1}) - (\ref{cond4}), respectively.
We provide the proofs of the lemmas in \ref{subsec:lemma_suff_cond_sign} to \ref{subsec:lemma_suff_cond_strict_eig}.

\begin{lemma}[Sufficient Conditions for (\ref{cond1})]
\label{lemma_suff_cond_sign}
If the following inequality holds:
\begin{equation*}
2\sqrt{2}\cdot\frac{K_1 + \rho s}{p(\lambda_1(\pmb{M}^*_{J,J}) - \lambda_2(\pmb{M}^*_{J,J}))}
\leq
\min_{i \in J}|u_{1,i}|,
\end{equation*}
then the condition (\ref{cond1}) holds with probability at least $1- (2s)^{-c}$ for any $c>0$.
Note that $K_1$ is defined in Lemma \ref{lemma_bernstein}.
\end{lemma}

\begin{lemma}[Sufficient Conditions for (\ref{cond2})]
\label{lemma_suff_cond_w}
If the following inequality holds:
$$
\rho
>
2\sqrt{ps^c  \cdot \big\{(1-p) \|\pmb{M}^*_{J^c,J}\|_F^2 + (d-s)s\sigma^2 \big\} } + p\cdot \|\pmb{M}^*_{J^c,J}\|_{\max},
$$
then the condition (\ref{cond2}) holds with probability at least $1-s^{-c}$ for any $c>0$.
\end{lemma}

\begin{lemma}[Sufficient Conditions for (\ref{cond3})]
\label{lemma_suff_cond_eig}
In addition to the conditions in Lemma \ref{lemma_suff_cond_w},
if the following inequality holds:
\begin{align*}
(K_2 + p\cdot \| \pmb{M}^*_{J^c,J} \|_2)^2\cdot (1+\sqrt{s})^2
&\leq
\Big\{ 
p\cdot (\lambda_1(\pmb{M}^*_{J,J}) - \lambda_2(\pmb{M}^*_{J,J}))
-
2\cdot K_1 - 2\rho s \Big\}
\\&~~~~\times
\Big\{
p\cdot (\lambda_1(\pmb{M}^*_{J,J}) - \lambda_1( \pmb{M}^*_{J^c,J^c} ))
- K_1 - K_3 - \rho d\Big\},
\end{align*}
then the condition (\ref{cond3}) holds with probability at least $1-s^{-c} - d^{-c} - (2s)^{-c} - (2(d-s))^{-c}$ for any $c>0$.
Note that $K_1$, $K_2$ and $K_3$ are defined in Lemma \ref{lemma_bernstein}.
\end{lemma}

\begin{lemma}[Sufficient Conditions for (\ref{cond4})]
\label{lemma_suff_cond_strict_eig}
In addition to the conditions in Lemma \ref{lemma_suff_cond_w},
if the following inequality holds:
$$
p\cdot (\lambda_1(\pmb{M}^*_{J,J})-\lambda_2(\pmb{M}^*_{J,J}))
\geq 2\cdot K_1 + 2\rho s,
$$
then the condition (\ref{cond4}) holds with probability at least $1-(2s)^{-c} $ for any $c>0$.
Note that $K_1$ is defined in Lemma \ref{lemma_bernstein}.
\end{lemma}

Since $\min_{i \in J}|u_{1,i}| < \sqrt{2}$ always holds, the sufficient condition in Lemma \ref{lemma_suff_cond_sign} implies the sufficient condition in Lemma \ref{lemma_suff_cond_strict_eig}.
Hence, we do not need the condition in Lemma \ref{lemma_suff_cond_strict_eig}.
Combining the sufficient condtions in Lemma \ref{lemma_suff_cond_sign} to Lemma \ref{lemma_suff_cond_eig} gives the result of Theorem \ref{thm_main}.

\section{Proof of Corollary \ref{cor_1}}

When $B = \sigma^2 =0$ and $p \geq 0.5$, we can write
\begin{align*}
K_1 
&\lesssim p \|\pmb{M}^*_{J,J}\|_{\max}\log s + \sqrt{p(1-p)} \|\pmb{M}^*_{J,J}\|_{2,\infty} \sqrt{\log s}
\\&=
p \log s \|\pmb{M}^*_{J,J}\|_{\max} \bigg(1 + \sqrt{\frac{1-p}{p}} (\mu_1 \sqrt{\log s})^{-1}\bigg).
\end{align*}
Under the conditions \eqref{cor1_cond1}, \eqref{cor1_cond4} and \eqref{cor1_cond5},
\begin{align*}
K_1 + \rho s
&\lesssim p \log s \|\pmb{M}^*_{J,J}\|_{\max} + \frac{p\bar{\lambda}(\pmb{M}^*_{J,J})}{s}
= p\log s \cdot \bar{\lambda}(\pmb{M}^*_{J,J}) \mu_0 + \frac{p\bar{\lambda}(\pmb{M}^*_{J,J})}{s}
\\& \ll p \bar{\lambda}(\pmb{M}^*_{J,J}) \log s \frac{1}{\sqrt{s}\log s} + \frac{p\bar{\lambda}(\pmb{M}^*_{J,J})}{\sqrt{s}}
= \frac{2p\bar{\lambda}(\pmb{M}^*_{J,J})}{\sqrt{s}}
\lesssim \frac{p\bar{\lambda}(\pmb{M}^*_{J,J}) \cdot \min_{i\in J}|u_{1,i}|}{2\sqrt{2}}.
\end{align*}
Hence, $2\sqrt{2}\cdot \frac{K_1 + \rho s}{p\bar{\lambda}(\pmb{M}^*_{J,J})}\leq \min_{i\in J}|u_{1,i}|$ holds asymptotically.

Next, under the conditions \eqref{cor1_cond2}, \eqref{cor1_cond4} and \eqref{cor1_cond5},
\begin{align*}
&2\sqrt{sp(1-p)}\|\pmb{M}^*_{J^c,J}\|_F + p\|\pmb{M}^*_{J^c,J}\|_{\max}
\leq 
p\|\pmb{M}^*_{J^c,J}\|_{\max} \Big(1 + 2\sqrt{s}\sqrt{\frac{1-p}{p}} \mu_2^{-1}\Big)
\\&\ll
p\|\pmb{M}^*_{J^c,J}\|_{\max} \bigg( \frac{\bar{\lambda}(\pmb{M}^*_{J,J})}{s^2 \|\pmb{M}^*_{J^c,J}\|_{\max}} + \frac{\bar{\lambda}(\pmb{M}^*_{J,J})}{s^2 \|\pmb{M}^*_{J^c,J}\|_{\max}} \bigg)
\lesssim
\frac{p\bar{\lambda}(\pmb{M}^*_{J,J})}{s^2} 
\lesssim \rho.
\end{align*}
Therefore, $\rho > 2\sqrt{sp(1-p)}\|\pmb{M}^*_{J^c,J}\|_F + p\|\pmb{M}^*_{J^c,J}\|_{\max}$ holds asymptotically.

Lastly, we will show that under the conditions \eqref{cor1_cond2}, \eqref{cor1_cond3}, \eqref{cor1_cond4} and \eqref{cor1_cond5},
\begin{align*}
(K_2 + p\cdot \| \pmb{M}^*_{J^c,J} \|_2)^2\cdot (1+\sqrt{s})^2
&\leq
\Big\{ 
p\cdot \bar{\lambda}(\pmb{M}^*_{J,J})
-
2\cdot K_1 - 2\rho s \Big\}
\\&~~~~\times
\Big\{
p\cdot \lambda_1(\pmb{M}^*_{J,J})
- K_1 - \rho s - K_3 - \frac{1}{\rho} \cdot 
( p\|\pmb{M}^*_{J,J^c} \|_{\infty, 2} + c_0\sqrt{d-s} )^2\Big\},
\end{align*}
where $c_0 = \sqrt{s p(1-p)}\|\pmb{M}^*_{J^c,J}\|_F$.
Here, the bound \eqref{eq:using_another_bound_w2} is used instead of \eqref{ineq_bound3}.
Under the conditions \eqref{cor1_cond2} and \eqref{cor1_cond4},
\begin{align*}
&(K_2 + p\cdot \|\pmb{M}^*_{J^c,J}\|_2)\cdot(1+\sqrt{s})
\\&\lesssim
\big(p \|\pmb{M}^*_{J^c,J}\|_{\max}\log d + \sqrt{p(1-p)} \max\big\{ \|\pmb{M}^*_{J^c,J}\|_{2,\infty}, \|{\pmb{M}^*}^\top_{J^c,J}\|_{2,\infty} \big\} \sqrt{\log d} + p \|\pmb{M}^*_{J^c,J}\|_2\big) \cdot \sqrt{s}
\\&\leq
p \|\pmb{M}^*_{J^c,J}\|_{\max}\cdot \big( \log d + \sqrt{\log d}\sqrt{\frac{1-p}{p}}\mu_2^{-1} + \mu_2^{-1} \big)\cdot\sqrt{s}
\\&\ll
p \|\pmb{M}^*_{J^c,J}\|_{\max}\cdot \bigg( \frac{\bar{\lambda}(\pmb{M}^*_{J,J}) \sqrt{s}}{s \|\pmb{M}^*_{J^c,J}\|_{\max}} + \frac{\bar{\lambda}(\pmb{M}^*_{J,J}) \sqrt{\log d}}{s\sqrt{s(d-s)} \|\pmb{M}^*_{J^c,J}\|_{\max}} + \frac{\bar{\lambda}(\pmb{M}^*_{J,J}) }{s \|\pmb{M}^*_{J^c,J}\|_{\max}} \bigg)\cdot\sqrt{s}
\\&\lesssim
p \|\pmb{M}^*_{J^c,J}\|_{\max}\cdot\frac{\bar{\lambda}(\pmb{M}^*_{J,J})}{\sqrt{s} \|\pmb{M}^*_{J^c,J}\|_{\max}}\cdot\sqrt{s}
= p\bar{\lambda}(\pmb{M}^*_{J,J}).
\end{align*}
Also, since $2\sqrt{2}\cdot \frac{K_1 + \rho s}{p\bar{\lambda}(\pmb{M}^*_{J,J})}\leq \min_{i\in J}|u_{1,i}| \leq \frac{1}{\sqrt{s}}$,
\begin{align*}
p\bar{\lambda}(\pmb{M}^*_{J,J}) - 2K_1 -2\rho s
\geq
p\bar{\lambda}(\pmb{M}^*_{J,J}) - \frac{p\bar{\lambda}(\pmb{M}^*_{J,J})}{\sqrt{2s}}
\simeq
p\bar{\lambda}(\pmb{M}^*_{J,J}).
\end{align*}
Moreover, under the conditions \eqref{cor1_cond3} and \eqref{cor1_cond4},
\begin{align*}
K_3 
&\lesssim
p \|\pmb{M}^*_{J^c,J^c}\|_{\max}\log(d-s) + \sqrt{p(1-p)} \|\pmb{M}^*_{J^c,J^c}\|_{2,\infty} \sqrt{\log(d-s)}
\\&=
 p \|\pmb{M}^*_{J^c,J^c}\|_{\max}\log(d-s)\cdot 
\bigg\{1 + \sqrt{\frac{1-p}{p}}\mu_3^{-1}\frac{1}{\sqrt{\log(d-s)}}\bigg\}
\\&\ll
p \|\pmb{M}^*_{J^c,J^c}\|_{\max}\log(d-s)\cdot 
\bigg\{ \frac{\bar{\lambda}(\pmb{M}^*_{J,J})}{\|\pmb{M}^*_{J^c,J^c}\|_{\max}\log(d-s)} + \frac{\bar{\lambda}(\pmb{M}^*_{J,J})}{\|\pmb{M}^*_{J^c,J^c}\|_{\max}\log(d-s)}\bigg\}
\simeq p \bar{\lambda}(\pmb{M}^*_{J,J}),
\end{align*}
and under the conditions \eqref{cor1_cond2}, \eqref{cor1_cond4} and \eqref{cor1_cond5},
\begin{align*}
&\frac{1}{\rho}\cdot ( p\|\pmb{M}^*_{J,J^c} \|_{\infty, 2} + \sqrt{s(d-s)p(1-p)}\|\pmb{M}^*_{J,J^c} \|_{F} )^2
\\&\leq
\frac{1}{\rho}\cdot \bigg\{ p\|\pmb{M}^*_{J,J^c} \|_{\max} \Big( \mu_2^{-1} + \sqrt{s(d-s)}\sqrt{\frac{1-p}{p}}\mu_2^{-1} \Big) \bigg\}^2
\\&\ll
\frac{1}{\rho}\cdot \bigg\{ p\|\pmb{M}^*_{J,J^c} \|_{\max}\Big( \frac{\bar{\lambda}(\pmb{M}^*_{J,J})}{s\|\pmb{M}^*_{J,J^c} \|_{\max}} + \frac{\bar{\lambda}(\pmb{M}^*_{J,J})}{s\|\pmb{M}^*_{J,J^c} \|_{\max}} \Big) \bigg\}^2
\simeq \frac{1}{\rho}\cdot \frac{(p \bar{\lambda}(\pmb{M}^*_{J,J}))^2}{s^2}
\\&\simeq
\frac{s^2}{p \bar{\lambda}(\pmb{M}^*_{J,J})}\cdot \frac{(p \bar{\lambda}(\pmb{M}^*_{J,J}))^2}{s^2} = p \bar{\lambda}(\pmb{M}^*_{J,J}).
\end{align*}
Hence,
\begin{align*}
&p\cdot \lambda_1(\pmb{M}^*_{J,J})
- K_1 - \rho s - K_3 - \frac{1}{\rho} \cdot 
( p\|\pmb{M}^*_{J,J^c} \|_{\infty, 2} + c_0\sqrt{d-s} )^2
\\&\geq
p \bar{\lambda}(\pmb{M}^*_{J,J})\cdot (1-o(1)) - \frac{p \bar{\lambda}(\pmb{M}^*_{J,J})}{2\sqrt{2s}}
\simeq
p \bar{\lambda}(\pmb{M}^*_{J,J}).
\end{align*}
Therefore, 
\begin{align*}
&(K_2 + p\cdot \| \pmb{M}^*_{J^c,J} \|_2)^2\cdot (1+\sqrt{s})^2
\ll
\big(p \bar{\lambda}(\pmb{M}^*_{J,J})\big)^2
\\&\lesssim
\Big\{ 
p\cdot \bar{\lambda}(\pmb{M}^*_{J,J})
-
2\cdot K_1 - 2\rho s \Big\}
\times
\Big\{
p\cdot \lambda_1(\pmb{M}^*_{J,J})
- K_1 - \rho s - K_3 - \frac{1}{\rho} \cdot 
( p\|\pmb{M}^*_{J,J^c} \|_{\infty, 2} + c_0\sqrt{d-s} )^2\Big\},
\end{align*}
that is, the desired result holds asymptotically.

\section{Proof of Corollary \ref{cor_2}}

Since
$\pmb{M}^*_{J^c, J} = \pmb{M}^*_{J, J^c} = \pmb{M}^*_{J^c, J^c} = 0$ and $\pmb{M}^*_{J, J} = \lambda_1(\pmb{M}^*) \pmb{u}_{1,J} \pmb{u}_{1,J}^\top$ for $\lambda_1(\pmb{M}^*) >0$, the conditions in Theorem \ref{thm_main} can be written as follows:
\begin{align}
2\sqrt{2}\cdot\frac{K_1 + \rho s}{p \lambda_1(\pmb{M}^*)}
&\leq
\min_{i \in J}|u_{1,i}|,
\label{cor_cond1}
\\
\rho
&>
2\sqrt{2}\cdot \sqrt{p\sigma^2 s^2 (d-s)},
\label{cor_cond2}
\\
K_2^2\cdot (1+\sqrt{s})^2
&\leq
\Big\{ 
p \lambda_1(\pmb{M}^*)
-
2\cdot K_1 - 2\rho s \Big\}
\times
\Big\{
p \lambda_1(\pmb{M}^*)
- K_1 - K_3 - \rho d\Big\},
\label{cor_cond3}
\end{align}
where $c = 1$.

First, when $B \leq (2p-1)\lambda_1(\pmb{M}^*)\cdot\max_{i,j\in J}|u_{1,i}u_{1,j}| = (2p-1)\|\pmb{M}^*_{J,J}\|_{\max}$ and $p\geq 0.5$, $K_1$ is expressed as
$$
K_1 = 2p\|\pmb{M}^*_{J,J}\|_{\max}\log(2s) + 2\big\{\sqrt{p(1-p)}\|\pmb{M}^*_{J,J}\|_{2,\infty} + \sqrt{ps\sigma^2}\big\}\sqrt{\log(2s)}.
$$
Hence, if the following inequalities hold:
\begin{align*}
2p\|\pmb{M}^*_{J,J}\|_{\max}\log(2s)
&\leq
\frac{1}{8\sqrt{2}}\cdot p \lambda_1(\pmb{M}^*)\cdot \min_{i \in J}|u_{1,i}|,
\\
2\sqrt{p(1-p)}\|\pmb{M}^*_{J,J}\|_{2,\infty}\cdot \sqrt{\log(2s)}
&\leq
\frac{1}{8\sqrt{2}}\cdot p \lambda_1(\pmb{M}^*)\cdot \min_{i \in J}|u_{1,i}|,
\\
2\sqrt{ps\sigma^2}\cdot \sqrt{\log(2s)}
&\leq
\frac{1}{8\sqrt{2}}\cdot p \lambda_1(\pmb{M}^*)\cdot \min_{i \in J}|u_{1,i}|,
\\
\rho s
&\leq
\frac{1}{8\sqrt{2}}\cdot p \lambda_1(\pmb{M}^*)\cdot \min_{i \in J}|u_{1,i}|,
\end{align*}
that is, if the following inequalities hold:
\begin{align}
\|\pmb{M}^*_{J,J}\|_{\max}
&\leq
\frac{1}{16\sqrt{2}\log(2s)}\cdot \lambda_1(\pmb{M}^*)\cdot \min_{i \in J}|u_{1,i}|,
\label{cor_cond4}
\\
\|\pmb{M}^*_{J,J}\|_{2,\infty}
&\leq
\frac{1}{16\sqrt{2}\sqrt{\log(2s)}}\cdot \sqrt{\frac{p}{1-p}}\cdot \lambda_1(\pmb{M}^*)\cdot \min_{i \in J}|u_{1,i}|,
\label{cor_cond5}
\\
2\sqrt{p\sigma^2s\log(2s)}
&\leq
\frac{1}{8\sqrt{2}}\cdot p \lambda_1(\pmb{M}^*)\cdot \min_{i \in J}|u_{1,i}|,
\label{cor_cond6}
\\
\rho
&\leq
\frac{1}{8\sqrt{2}s}\cdot p \lambda_1(\pmb{M}^*)\cdot \min_{i \in J}|u_{1,i}|,
\label{cor_cond7}
\end{align}
then (\ref{cor_cond1}) holds.
Note that (\ref{cor_cond2}) and (\ref{cor_cond7}) imply that
$$
2\sqrt{2}\cdot \sqrt{p\sigma^2 s^2 (d-s)}
< \frac{1}{8\sqrt{2}s}\cdot p \lambda_1(\pmb{M}^*)\cdot \min_{i \in J}|u_{1,i}|.
$$
Since $2s^3(d-s) \geq \log(2s)$ for any $s\geq 1$ and $d>s$,
(\ref{cor_cond2}) and (\ref{cor_cond7}) are sufficient for (\ref{cor_cond6}).

Now, we will derive the sufficient conditions for (\ref{cor_cond3}).
First, by the conditions (\ref{cor_cond1}) and (\ref{cor_cond7}), we have that 
\begin{align*}
2K_1 + 2\rho s 
&\leq 
\frac{1}{\sqrt{2}} \cdot p \lambda_1(\pmb{M}^*) \min_{i \in J}|u_{1,i}|,
\\
K_1 + \rho d 
&= K_1 + \rho s + \rho (d-s)
\leq
\frac{1}{2\sqrt{2}} \cdot p \lambda_1(\pmb{M}^*) \min_{i \in J}|u_{1,i}|
+
\frac{d-s}{8\sqrt{2}s}\cdot p \lambda_1(\pmb{M}^*)\cdot \min_{i \in J}|u_{1,i}|.
\end{align*}
Also, since
\begin{gather*}
2B 
\leq 
2(2p-1)\|\pmb{M}^*_{J,J}\|_{\max} 
\leq 
(2p-1) \frac{1}{8\sqrt{2}\log(2s)}\cdot \lambda_1(\pmb{M}^*)\cdot \min_{i \in J}|u_{1,i}|
=
(2-\frac{1}{p})\cdot \frac{p\lambda_1(\pmb{M}^*)\cdot \min_{i \in J}|u_{1,i}|}{8\sqrt{2}\log(2s)},
\\
2\sqrt{p\sigma^2}
< 
\frac{1}{16 s^2\sqrt{d-s}}\cdot p \lambda_1(\pmb{M}^*)\cdot \min_{i \in J}|u_{1,i}|,
\end{gather*}
we can state
\begin{align*}
K_2
&=
2B \log d + 2\sqrt{p\sigma^2 \max\{d-s, s\}}\cdot \sqrt{\log d}
\\&\leq
p\lambda_1(\pmb{M}^*)\cdot \min_{i \in J}|u_{1,i}|\cdot
\Bigg\{
\underbrace{
(2-\frac{1}{p})\cdot \frac{\log d}{8\sqrt{2}\log(2s)}
+
\frac{\sqrt{\max\{d-s, s\}}\cdot \sqrt{\log d}}{16 s^2\sqrt{d-s}}
}_{=:a_1}
\Bigg\}
\end{align*}
and
\begin{align*}
K_3 
&= 
2B\log(2(d-s)) + 2 \sqrt{p\sigma^2(d-s)} \sqrt{\log(2(d-s))}
\\&\leq
p\lambda_1(\pmb{M}^*)\cdot \min_{i \in J}|u_{1,i}|\cdot
\Bigg\{
\underbrace{
(2-\frac{1}{p})\cdot \frac{\log(2(d-s))}{8\sqrt{2}\log(2s)}
+
\frac{\sqrt{\log(2(d-s))}}{16 s^2}
}_{=:a_2}
\Bigg\}.
\end{align*}
Therefore, the following inequality is sufficient for (\ref{cor_cond3}):
\begin{multline*}
\bigg(p\lambda_1(\pmb{M}^*)\cdot\min_{i \in J}|u_{1,i}| 
\cdot a_1 \cdot (1+\sqrt{s}) \bigg)^2
\leq
\Big\{ 
p \lambda_1(\pmb{M}^*)
-
\frac{1}{\sqrt{2}}\cdot p\lambda_1(\pmb{M}^*)\cdot\min_{i \in J}|u_{1,i}| 
\Big\}
\\ 
\times
\Big\{
p \lambda_1(\pmb{M}^*)
-
\frac{1}{2\sqrt{2}} \cdot p \lambda_1(\pmb{M}^*) \min_{i \in J}|u_{1,i}|
-
\frac{d-s}{8\sqrt{2}s}\cdot p \lambda_1(\pmb{M}^*)\cdot \min_{i \in J}|u_{1,i}|
- a_2 \cdot p \lambda_1(\pmb{M}^*)\cdot\min_{i \in J}|u_{1,i}|
\Big\}
\end{multline*}
\begin{align*}
\Leftrightarrow
\bigg(a_1 \cdot (1+\sqrt{s})\cdot\min_{i \in J}|u_{1,i}| \bigg)^2
&\leq
\Big\{ 
1
-
\frac{1}{\sqrt{2}}\cdot\min_{i \in J}|u_{1,i}| 
\Big\}
\times
\Big\{
1
- 
\bigg(
\frac{1}{2\sqrt{2}}
+ \frac{d-s}{8\sqrt{2}s}
+ a_2
\bigg)
\cdot\min_{i \in J}|u_{1,i}| 
\Big\}.
\end{align*}
Note that the quadratic inequality $(a x)^2 \leq (1-bx)(1-cx)$
holds if
$a^2 \neq bc$ and $0 \leq x \leq \frac{b+c-\sqrt{(b-c)^2 + 4a^2}}{2(bc-a^2)}$.
By using this fact, we have that if the following inequality holds:
\begin{align*}
\min_{i \in J}|u_{1,i}|
&\leq
\frac
{12 + \frac{d-s}{s} + 8\sqrt{2}a_2 - \sqrt{ (4-\frac{d-s}{s} - 8\sqrt{2}a_2)^2 + 512 a_1^2 (1+\sqrt{s})^2 }}
{4\sqrt{2} + \sqrt{2}\cdot \frac{d-s}{s} + 16a_2 - 16\sqrt{2}a_1^2(1+\sqrt{s})^2},
\end{align*}
then (\ref{cor_cond3}) holds.
Since $\min_{i \in J}|u_{1,i}| \leq \frac{1}{\sqrt{s}}$,
the following inequality is sufficient for (\ref{cor_cond3}):
\begin{align*}
\frac{1}{\sqrt{s}}
&\leq
\frac
{12 + \frac{d-s}{s} + 8\sqrt{2}a_2 - \sqrt{ (4-\frac{d-s}{s} - 8\sqrt{2}a_2)^2 + 512 a_1^2 (1+\sqrt{s})^2 }}
{4\sqrt{2} + \sqrt{2}\cdot \frac{d-s}{s} + 16a_2 - 16\sqrt{2}a_1^2(1+\sqrt{s})^2}.
\end{align*}

In sum, if the following inequalities hold:
\begin{align*}
\|\pmb{M}^*_{J,J}\|_{\max}
&\leq
\frac{1}{16\sqrt{2}\log(2s)}\cdot \lambda_1(\pmb{M}^*)\cdot \min_{i \in J}|u_{1,i}|,
\\
\|\pmb{M}^*_{J,J}\|_{2,\infty}
&\leq
\frac{1}{16\sqrt{2}\sqrt{\log(2s)}}\cdot \sqrt{\frac{p}{1-p}}\cdot \lambda_1(\pmb{M}^*)\cdot \min_{i \in J}|u_{1,i}|,
\\
B 
&\leq 
(2p-1)\lambda_1(\pmb{M}^*)\cdot\max_{i,j\in J} |u_{1,i} u_{1,j}|,
\\
\rho
&\leq
\frac{1}{8\sqrt{2}s}\cdot p \lambda_1(\pmb{M}^*)\cdot \min_{i \in J}|u_{1,i}|,
\\
\rho
&>
2\sqrt{2}\cdot \sqrt{p\sigma^2 s^2 (d-s)},
\\
\frac{1}{\sqrt{s}}
&\leq
\frac
{12 + \frac{d-s}{s} + 8\sqrt{2}a_2 - \sqrt{ (4-\frac{d-s}{s} - 8\sqrt{2}a_2)^2 + 512 a_1^2 (1+\sqrt{s})^2 }}
{4\sqrt{2} + \sqrt{2}\cdot \frac{d-s}{s} + 16a_2 - 16\sqrt{2}a_1^2(1+\sqrt{s})^2},
\end{align*}
then the desired result holds, where
$
a_1 
=
(2-\frac{1}{p})\cdot \frac{\log d}{8\sqrt{2}\log(2s)}
+
\frac{\sqrt{\max\{d-s, s\}}\cdot \sqrt{\log d}}{16 s^2\sqrt{d-s}}$
and
$a_2
=
(2-\frac{1}{p})\cdot \frac{\log(2(d-s))}{8\sqrt{2}\log(2s)}
+
\frac{\sqrt{\log(2(d-s))}}{16 s^2}.$
Since 
$\|\pmb{M}^*\|_{\max} = \lambda_1(\pmb{M}^*)\max_{i,j\in J} |u_{1,i} u_{1,j}|$
and
$\|\pmb{M}^*\|_{2,\infty} = \lambda_1(\pmb{M}^*)\max_{i\in J}\sqrt{\sum_{j\in J} u_{1,i}^2 u_{1,j}^2} = \lambda_1(\pmb{M}^*)\max_{i\in J} |u_{1,i}|$, the first two conditions can be written as
$\frac{\max_{i,j\in J} |u_{1,i} u_{1,j}|}{\min_{i \in J}|u_{1,i}|}
\leq
\frac{1}{16\sqrt{2}\log(2s)}$
and
$\frac{\max_{i\in J} |u_{1,i}|}{\min_{i \in J}|u_{1,i}|}
\leq
\frac{1}{16\sqrt{2}\sqrt{\log(2s)}}\cdot \sqrt{\frac{p}{1-p}}$.

\section{Other Proofs}

\subsection{Proof of Proposition \ref{prop_fps_kkt}}
\label{subsec:pf_prop_fps_kkt}

With the primal variable $\pmb{X}\in\mathbb{R}^{d\times d}$ and the dual variables $\pmb{Z}\in\mathbb{R}^{d\times d}$, $\pmb{\Lambda}\in\mathbb{R}^{d\times d}$ and $\mu \in \mathbb{R}$,
the Lagrangian of the problem (\ref{fps_problem}) is written as
$$
L(\pmb{X}, \pmb{Z}, \pmb{\Lambda}, \mu)
= 
- \langle \pmb{M}, \pmb{X} \rangle + \rho \langle \pmb{X}, \pmb{Z} \rangle
- \langle \pmb{\Lambda}, \pmb{X} \rangle  + \mu \cdot (tr(\pmb{X}) - 1)
$$
where $Z_{ij} \in \partial |X_{ij}|$ for each $i,j \in [d]$.
According to the standard KKT condition, we can derive that
$(\hat{\pmb{X}}, \hat{\pmb{Z}}, \hat{\pmb{\Lambda}}, \hat{\mu})$ is optimal
if and only if the followings hold:
\begin{itemize}
\item Primal feasibility: 
$\hat{\pmb{X}} \succeq 0$, $tr(\hat{\pmb{X}}) = 1$
\item Dual feasibility: 
$\hat{\pmb{\Lambda}} \succeq 0$,
$\hat{Z}_{ij} \in \partial |\hat{X}_{ij}|$ for each $i,j \in [d]$
\item Complementary slackness: 
$\langle \hat{\pmb{\Lambda}}, \hat{\pmb{X}} \rangle = 0$
($\Leftrightarrow \hat{\pmb{\Lambda}} \hat{\pmb{X}} = 0$ if $\hat{\pmb{X}} \succeq 0$ and $\hat{\pmb{\Lambda}} \succeq 0$)
\item Stationarity: 
$\hat{\pmb{\Lambda}} = - \pmb{M} + \rho  \hat{\pmb{Z}} + \hat\mu \cdot \pmb{I}$.
\end{itemize}
By substituting $\hat{\pmb{\Lambda}}$ with $- \pmb{M} + \rho  \hat{\pmb{Z}} + \hat\mu \cdot \pmb{I}$,
it can be shown that the above conditions are equivalent to
\begin{align*}
&\hat{\pmb{X}} \succeq 0, tr(\hat{\pmb{X}}) = 1
\\&
\pmb{M} - \rho  \hat{\pmb{Z}} \preceq \hat\mu \pmb{I}
\\&
\hat{Z}_{ij} \in \partial |\hat{X}_{ij}| ~~\text{ for each } i,j \in [d]
\\&
(\pmb{M} - \rho  \hat{\pmb{Z}})\hat{\pmb{X}} = \hat\mu \cdot \hat{\pmb{X}}.
\end{align*}

To use the primal-dual witness construction, we now consider the following restricted problem:

\begin{equation}
\label{fps_restricted_problem}
\underset{\pmb{X}\succeq 0, tr(\pmb{X}) = 1 \text{ and } supp(\pmb{X})\subseteq J\times J}{\max}~ 
\langle \pmb{M}, \pmb{X} \rangle - \rho \|\pmb{X}\|_{1,1}.
\end{equation}
Similarly to the above, we can derive that
$\hat{\pmb{X}} = 
\begin{pmatrix}
\hat{\pmb{X}}_{J,J} & 0 \\
0 & 0
\end{pmatrix}
$
is optimal to the problem (\ref{fps_restricted_problem}) if and only if
\begin{align*}
&\hat{\pmb{X}}_{J,J} \succeq 0, tr(\hat{\pmb{X}}_{J,J}) = 1
\\&
\pmb{M}_{J,J} - \rho  \hat{\pmb{Z}}_{J,J} \preceq \hat\mu \pmb{I}
\\&
\hat{Z}_{ij} \in \partial |\hat{X}_{ij}| ~~\text{ for each } i,j \in J
\\&
(\pmb{M}_{J,J} - \rho  \hat{\pmb{Z}}_{J,J})\hat{\pmb{X}}_{J,J} = \hat\mu \cdot \hat{\pmb{X}}_{J,J}.
\end{align*}

Now, we want for the above solution $\hat{\pmb{X}} = \begin{pmatrix}
\hat{\pmb{X}}_{J,J} & 0 \\
0 & 0
\end{pmatrix}$
to satisfy the optimality conditions of the original problem (\ref{fps_problem}).
Furthermore, by assuming the strict dual feasibility, we want to guarantee $supp(diag(\hat{\pmb{X}})) \subseteq J$.
We can easily derive their sufficient conditions listed below:
\begin{align*}
&\hat{\pmb{X}}_{J,J} \succeq 0, tr(\hat{\pmb{X}}_{J,J}) = 1
\\&
\pmb{M}_{J,J} - \rho  \hat{\pmb{Z}}_{J,J} \preceq \hat\mu \pmb{I}
\\&
\pmb{M} - \rho  \hat{\pmb{Z}} \preceq \hat\mu \pmb{I}
\\&
\hat{Z}_{ij} \in \partial |\hat{X}_{ij}| ~~\text{ for each } (i,j) \in J\times J
\\&
\hat{Z}_{ij} \in (-1, 1) ~~\text{ for each } (i,j) \notin J\times J
\\&
(\pmb{M}_{J,J} - \rho  \hat{\pmb{Z}}_{J,J})\hat{\pmb{X}}_{J,J} = \hat\mu \cdot \hat{\pmb{X}}_{J,J}
\\&
(\pmb{M}_{J^c,J} - \rho  \hat{\pmb{Z}}_{J^c,J})\hat{\pmb{X}}_{J,J} = 0.
\end{align*}
If the above conditions hold, then 
$\hat{\pmb{X}} = \begin{pmatrix}
\hat{\pmb{X}}_{J,J} & 0 \\
0 & 0
\end{pmatrix}$ is optimal to the problem (\ref{fps_problem}) and satisfies $supp(diag(\hat{\pmb{X}})) \subseteq J$.

\subsection{Proof of Proposition \ref{prop_kkt}}
\label{subsec:pf_prop_kkt}

With the primal variable $\pmb{x}\in\mathbb{R}^d$ and the dual variables $\pmb{z} \in \mathbb{R}^d$ and $\lambda \in \mathbb{R}$,
the Lagrangian of the problem (\ref{noncvx_problem}) is written as
$$
L(\pmb{x}, \pmb{z}, \lambda)
= - \pmb{x}^\top \pmb{M} \pmb{x} + \rho \langle \pmb{x}, \pmb{z} \rangle^2 + \lambda (\pmb{x}^\top \pmb{x} - 1)
= \pmb{x}^\top ( -\pmb{M} + \rho \pmb{z}\pmb{z}^\top + \lambda \pmb{I} ) \pmb{x} - \lambda
$$
where $z_i \in \partial|x_i|$ for $i \in [d]$.
By denoting the primal solution by $\tilde{\pmb{x}} = (\tilde{\pmb{x}}_1^\top, \tilde{\pmb{x}}_2^\top)^\top \in \mathbb{R}^{s}\times \mathbb{R}^{d-s}$ and the dual solutions by $(\tilde{\pmb{z}}^\top, \tilde{\pmb{w}}^\top)^\top \in \mathbb{R}^{s}\times \mathbb{R}^{d-s}$ and $\tilde{\lambda}\in \mathbb{R}$,
the KKT conditions of (\ref{noncvx_problem}) are given as follows:
\begin{itemize}
\item Primal feasibility: 
$\tilde{\pmb{x}}_1^\top\tilde{\pmb{x}}_1 + \tilde{\pmb{x}}_2^\top\tilde{\pmb{x}}_2 = 1$
\item Dual feasibility:
$
\begin{cases}
\tilde{z}_i= \text{sign}(\tilde{x}_{1,i}) &\text{if~} \tilde{x}_{1,i} \neq 0 \\
\tilde{z}_i\in [-1, 1] &\text{if~} \tilde{x}_{1,i} = 0,
\end{cases}
$~~
$
\begin{cases}
\tilde{w}_i = \text{sign}(\tilde{x}_{2,i}) &\text{if~} \tilde{x}_{2,i} \neq 0 \\
\tilde{w}_i \in [-1, 1] &\text{if~} \tilde{x}_{2,i} = 0
\end{cases}
$
\item Stationarity: 
$\Big[ \pmb{M}-\rho (\tilde{\pmb{z}}^\top, \tilde{\pmb{w}}^\top)^\top (\tilde{\pmb{z}}^\top, \tilde{\pmb{w}}^\top) \Big]
\begin{pmatrix}
\tilde{\pmb{x}}_1 \\ \tilde{\pmb{x}}_2
\end{pmatrix}
= \tilde{\lambda} 
\begin{pmatrix}
\tilde{\pmb{x}}_1 \\ \tilde{\pmb{x}}_2
\end{pmatrix}$
\\
$\tilde{\lambda}\pmb{I}
\succeq \pmb{M}-\rho (\tilde{\pmb{z}}^\top, \tilde{\pmb{w}}^\top)^\top (\tilde{\pmb{z}}^\top, \tilde{\pmb{w}}^\top)$.
\end{itemize}
Here, it can be easily checked that the primal feasibility and the stationarity conditions are equivalent to the following:
$$
\begin{pmatrix}
\tilde{\pmb{x}}_1 \\ \tilde{\pmb{x}}_2
\end{pmatrix}
\text{~is the leading eigenvector of~} 
\pmb{M}-\rho (\tilde{\pmb{z}}^\top, \tilde{\pmb{w}}^\top)^\top (\tilde{\pmb{z}}^\top, \tilde{\pmb{w}}^\top).
$$

To proceed with primal-dual witness argument,
we now consider the KKT conditions for the problem (\ref{noncvx_problem}) with an additional constraint $supp(({\pmb{x}}^\top, {\pmb{y}}^\top)^\top )\subseteq J$, that is, $\pmb{y} = \pmb{0}$.
With the primal solution $\hat{\pmb{x}} \in \mathbb{R}^{s}$ and the dual solution $\hat{\pmb{z}} \in \mathbb{R}^{s}$, the KKT conditions are given by
\begin{align*}
& \hat{\pmb{x}} \text{~is the leading eigenvector of~} \pmb{M}_{J,J}-\rho \hat{\pmb{z}} \hat{\pmb{z}}^\top
\\ & 
\hat{z}_i = \text{sign}(\hat{x}_i) \text{~~~~if~} \hat{x}_{i} \neq 0,
~~\hat{z}_i \in [-1, 1] \text{~~~~if~} \hat{x}_{i} = 0
\end{align*}

Now we will show that if the following conditions hold,
the solution $({\hat{\pmb{x}}}^\top, {\pmb{0}}^\top)^\top$
satisfies the KKT conditions of the original problem (\ref{noncvx_problem}) and the strict dual feasibility:
\begin{align}
& \hat{\pmb{x}} \text{~is the leading eigenvector of~} \pmb{M}_{J,J}-\rho \hat{\pmb{z}} \hat{\pmb{z}}^\top
\label{eq_kkt1}
\\ &
\hat{z}_i = \text{sign}(\hat{x}_i) \text{~~~~if~} \hat{x}_{i} \neq 0,
~~\hat{z}_i \in [-1, 1] \text{~~~~if~} \hat{x}_{i} = 0
\nonumber
\\ & 
\hat{\pmb{w}} = \frac{1}{\rho \|\hat{\pmb{x}}\|_1} \pmb{M}_{J^c, J}\hat{\pmb{x}}
\label{eq_kkt2}
\\ & 
\|\hat{\pmb{w}}\|_\infty < 1.
\nonumber
\\ &
\lambda_1(\pmb{M}_{J,J}-\rho \hat{\pmb{z}} \hat{\pmb{z}}^\top) = \lambda_1(\pmb{M}-\rho (\hat{\pmb{z}}^\top, \hat{\pmb{w}}^\top)^\top (\hat{\pmb{z}}^\top, \hat{\pmb{w}}^\top)).
\nonumber
\end{align}
Let $\hat{\lambda} 
= \lambda_1\big(\pmb{M}_{J,J}-\rho \hat{\pmb{z}} \hat{\pmb{z}}^\top\big)
= \lambda_1\big(\pmb{M}-\rho (\hat{\pmb{z}}^\top, \hat{\pmb{w}}^\top)^\top (\hat{\pmb{z}}^\top, \hat{\pmb{w}}^\top)\big)$.
For $({\hat{\pmb{x}}}^\top, {\pmb{0}}^\top)^\top$ to be the leading eigenvector of $\pmb{M}-\rho (\hat{\pmb{z}}^\top, \hat{\pmb{w}}^\top)^\top (\hat{\pmb{z}}^\top, \hat{\pmb{w}}^\top)$,
\begin{equation}
\Big[ \pmb{M}-\rho (\hat{\pmb{z}}^\top, \hat{\pmb{w}}^\top)^\top (\hat{\pmb{z}}^\top, \hat{\pmb{w}}^\top) \Big]
\begin{pmatrix}
\hat{\pmb{x}} \\ \pmb{0}
\end{pmatrix}
= \hat{\lambda} 
\begin{pmatrix}
\hat{\pmb{x}} \\ \pmb{0}
\end{pmatrix}
\label{eq_kkt3}
\end{equation}
needs to be satisfied.
It can be easily checked that (\ref{eq_kkt3}) is equivalent to (\ref{eq_kkt1}) and (\ref{eq_kkt2}).

Now, let 
$\hat{\pmb{X}} = \begin{pmatrix}
\hat{\pmb{x}} \hat{\pmb{x}}^\top & 0 \\
0 & 0
\end{pmatrix}$,
$\hat{\pmb{Z}} = (\hat{\pmb{z}}^\top, \hat{\pmb{w}}^\top)^\top (\hat{\pmb{z}}^\top, \hat{\pmb{w}}^\top)$
and $\hat{\mu}= \lambda_1\big(\pmb{M}_{J,J}-\rho \hat{\pmb{z}} \hat{\pmb{z}}^\top\big)
= \lambda_1\big(\pmb{M}-\rho (\hat{\pmb{z}}^\top, \hat{\pmb{w}}^\top)^\top (\hat{\pmb{z}}^\top, \hat{\pmb{w}}^\top)\big)$.
Then it can be easily shown that $(\hat{\pmb{X}}, \hat{\pmb{Z}}, \hat{\mu})$
satisfies the sufficient conditions in Proposition \ref{prop_fps_kkt}.
That is, $\hat{\pmb{X}}$ constructed above
is an optimal solution to the problem (\ref{fps_problem})
and satisfies $supp(diag(\hat{\pmb{X}})) \subseteq J$.
To ensure that there is no false positive, we consider the additional condition that $\text{sign}(\hat{x}_i) = \text{sign}(u_{1,i})$ for all $i\in J$.

Lastly, for the uniqueness, we need an additional condition presented in the following lemma.

\begin{lemma}
For 
$\hat{\pmb{X}} = \begin{pmatrix}
\hat{\pmb{x}} \hat{\pmb{x}}^\top & 0 \\
0 & 0
\end{pmatrix}$
and
$\hat{\pmb{Z}} = (\hat{\pmb{z}}^\top, \hat{\pmb{w}}^\top)^\top (\hat{\pmb{z}}^\top, \hat{\pmb{w}}^\top)$
constructed above,
if the following condition holds:
$$
\lambda_1(\pmb{M}_{J,J}-\rho \hat{\pmb{z}} \hat{\pmb{z}}^\top )
> \lambda_2(\pmb{M}_{J,J}-\rho \hat{\pmb{z}} \hat{\pmb{z}}^\top)
$$
then the solution $\hat{\pmb{X}}$ is a unique optimal solution to the problem (\ref{fps_problem}).
\end{lemma}

\begin{proof}
According to the standard primal-dual witness construction, we only need to show that under the condition, $\hat{\pmb{X}}_{J,J} = \hat{\pmb{x}} \hat{\pmb{x}}^\top$ is a unique optimal solution to the restricted problem (\ref{fps_restricted_problem}).

Assume that there exists another optimal solution to the problem (\ref{fps_restricted_problem}), say $\tilde{\pmb{X}}_{J,J}$.
Also, denote its dual optimal solution by $\tilde{\pmb{Z}}_{J,J}$. 
Then, we can write
\begin{align*}
& \langle \pmb{M}_{J,J}, \hat{\pmb{X}}_{J,J} \rangle - \rho \|\hat{\pmb{X}}_{J,J}\|_{1,1}
=
\langle \pmb{M}_{J,J}-\rho \hat{\pmb{z}} \hat{\pmb{z}}^\top,  \hat{\pmb{x}} \hat{\pmb{x}}^\top \rangle
= \hat{\pmb{x}}^\top (\pmb{M}_{J,J}-\rho \hat{\pmb{z}} \hat{\pmb{z}}^\top) \hat{\pmb{x}}
\\ &
=\langle \pmb{M}_{J,J}, \tilde{\pmb{X}}_{J,J} \rangle - \rho \|\tilde{\pmb{X}}_{J,J}\|_{1,1}
= \langle \pmb{M}_{J,J}-\rho \tilde{\pmb{Z}}_{J,J}, \tilde{\pmb{X}}_{J,J} \rangle.
\end{align*}

Recall that $\hat{\pmb{x}}$ is the leading eigenvector of $\pmb{M}_{J,J}-\rho \hat{\pmb{z}} \hat{\pmb{z}}^\top$, that is, 
$\hat{\pmb{x}}^\top (\pmb{M}_{J,J}-\rho \hat{\pmb{z}} \hat{\pmb{z}}^\top) \hat{\pmb{x}} = \lambda_1(\pmb{M}_{J,J}-\rho \hat{\pmb{z}} \hat{\pmb{z}}^\top)$.
Now, we will show that 
$\langle \pmb{M}_{J,J}-\rho \hat{\pmb{z}} \hat{\pmb{z}}^\top, \tilde{\pmb{X}}_{J,J} \rangle < \lambda_1(\pmb{M}_{J,J}-\rho \hat{\pmb{z}} \hat{\pmb{z}}^\top)$
for any matrix $\tilde{\pmb{X}}_{J,J} \neq \hat{\pmb{x}} \hat{\pmb{x}}^\top$ such that $\tilde{\pmb{X}}_{J,J}\succeq 0$ and $tr(\tilde{\pmb{X}}_{J,J}) = 1$.
Let $\tilde{\pmb{X}}_{J,J} = \sum_{i\in J}\theta_i \pmb{v}_i \pmb{v}_i^\top$, which is the spectral decomposition of $\tilde{\pmb{X}}_{J,J}$.
We can derive that
\begin{align*}
\langle \pmb{M}_{J,J}-\rho \hat{\pmb{z}} \hat{\pmb{z}}^\top, \tilde{\pmb{X}}_{J,J} \rangle
&= 
\langle \pmb{M}_{J,J}-\rho \hat{\pmb{z}} \hat{\pmb{z}}^\top, \sum_{i\in J}\theta_i \pmb{v}_i \pmb{v}_i^\top \rangle
=
\sum_{i\in J} \theta_i \pmb{v}_i^\top (\pmb{M}_{J,J}-\rho \hat{\pmb{z}} \hat{\pmb{z}}^\top) \pmb{v}_i
\leq
\lambda_1(\pmb{M}_{J,J}-\rho \hat{\pmb{z}} \hat{\pmb{z}}^\top)
\end{align*}
where the last inequality holds since $\sum_{i\in J} \theta_i = tr(\tilde{\pmb{X}}_{J,J}) = 1$ and 
$\pmb{v}_i^\top (\pmb{M}_{J,J}-\rho \hat{\pmb{z}} \hat{\pmb{z}}^\top) \pmb{v}_i\leq \lambda_1(\pmb{M}_{J,J}-\rho \hat{\pmb{z}} \hat{\pmb{z}}^\top)$.
Here, the equality holds only if $\theta_1 = 1$, $\theta_i = 0$ for $i\neq 1$ and $\pmb{v}_1 = \hat{\pmb{x}}$, that is, 
$\tilde{\pmb{X}}_{J,J} = \hat{\pmb{x}} \hat{\pmb{x}}^\top$.
Therefore, 
$\langle \pmb{M}_{J,J}-\rho \hat{\pmb{z}} \hat{\pmb{z}}^\top, \tilde{\pmb{X}}_{J,J} \rangle < \lambda_1(\pmb{M}_{J,J}-\rho \hat{\pmb{z}} \hat{\pmb{z}}^\top)$
for any matrix $\tilde{\pmb{X}}_{J,J} \neq \hat{\pmb{x}} \hat{\pmb{x}}^\top$ such that $\tilde{\pmb{X}}_{J,J}\succeq 0$ and $tr(\tilde{\pmb{X}}_{J,J}) = 1$.

With this fact, we can derive that
\begin{align*}
\langle \pmb{M}_{J,J}, \hat{\pmb{X}}_{J,J} \rangle - \rho \|\hat{\pmb{X}}_{J,J}\|_{1,1}
&= 
\hat{\pmb{x}}^\top (\pmb{M}_{J,J}-\rho \hat{\pmb{z}} \hat{\pmb{z}}^\top) \hat{\pmb{x}}
= \lambda_1(\pmb{M}_{J,J}-\rho \hat{\pmb{z}} \hat{\pmb{z}}^\top)
\\&>
\langle \pmb{M}_{J,J}-\rho \hat{\pmb{z}} \hat{\pmb{z}}^\top, \tilde{\pmb{X}}_{J,J} \rangle
=
\langle \pmb{M}_{J,J}-\rho \tilde{\pmb{Z}}_{J,J}, \tilde{\pmb{X}}_{J,J} \rangle
+ \rho \langle \tilde{\pmb{Z}}_{J,J}-\hat{\pmb{z}} \hat{\pmb{z}}^\top, \tilde{\pmb{X}}_{J,J} \rangle
\\&=
\langle \pmb{M}_{J,J}, \tilde{\pmb{X}}_{J,J} \rangle - \rho \|\tilde{\pmb{X}}_{J,J}\|_{1,1} 
+ \rho \langle \tilde{\pmb{Z}}_{J,J}-\hat{\pmb{z}} \hat{\pmb{z}}^\top, \tilde{\pmb{X}}_{J,J} \rangle.
\end{align*}
Since $\langle \pmb{M}_{J,J}, \hat{\pmb{X}}_{J,J} \rangle - \rho \|\hat{\pmb{X}}_{J,J}\|_{1,1} = \langle \pmb{M}_{J,J}, \tilde{\pmb{X}}_{J,J} \rangle - \rho \|\tilde{\pmb{X}}_{J,J}\|_{1,1} $ by assumption, 
the above inequality implies 
$\langle \tilde{\pmb{Z}}_{J,J}-\hat{\pmb{z}} \hat{\pmb{z}}^\top, \tilde{\pmb{X}}_{J,J} \rangle <0$, that is, 
$\langle \tilde{\pmb{Z}}_{J,J}, \tilde{\pmb{X}}_{J,J} \rangle
< 
\langle \hat{\pmb{z}} \hat{\pmb{z}}^\top, \tilde{\pmb{X}}_{J,J} \rangle$.
This contradicts the fact that 
$\langle \tilde{\pmb{Z}}_{J,J}, \tilde{\pmb{X}}_{J,J} \rangle = \sup_{\|\pmb{Z}_{J,J}\|_{\max} \leq 1} \langle \pmb{Z}_{J,J}, \tilde{\pmb{X}}_{J,J} \rangle$,
and thus the desired result holds.

\end{proof}

\subsection{Lemma \ref{lemma_bernstein}}
\label{subsec:lemma_bernstein}

The following lemma is frequently used in the proofs.

\begin{lemma}
\label{lemma_bernstein}
For any $c>0$,
\begin{align*}
\| \mathbb{E}[\pmb{M}_{J,J}] - \pmb{M}_{J,J} \|_2 
&\leq (c+1)\cdot R_1 \log (2s) + \sqrt{2(c+1)}\cdot R_2 \sqrt{\log (2s)} =: K_1
\\
&~~~~~~~~~~~~~~~~~~~~~~~
~~~~~~~~~~~~~~~~~~~~~~~~\text{ with probability at least } 1- (2s)^{-c},
\\
\| \mathbb{E}[\pmb{M}_{J^c,J}] - \pmb{M}_{J^c,J} \|_2 
&\leq (c+1)\cdot R_3 \log d + \sqrt{2(c+1)}\cdot R_4 \sqrt{\log d} =: K_2
\\
&~~~~~~~~~~~~~~~~~~~~~~~
~~~~~~~~~~~~~~~~~~~~~~~~\text{ with probability at least } 1- d^{-c},
\\
\| \mathbb{E}[\pmb{M}_{J^c,J^c}] - \pmb{M}_{J^c,J^c} \|_2 
&\leq (c+1)\cdot R_5 \log(2(d-s)) + \sqrt{2(c+1)}\cdot R_6 \sqrt{\log(2(d-s))} =: K_3
\\
&~~~~~~~~~~~~~~~~~~~~~~~
~~~~~~~~~~~~~~~~~~~~~~~~
\text{ with probability at least } 1- (2(d-s))^{-c}
\end{align*}
where
\begin{align*}
R_1 
&:= \max \{ (1-p) \|\pmb{M}_{J,J}^*\|_{\max} + B,~~ p\|\pmb{M}_{J,J}^*\|_{\max} \},
\\
R_2 
&:= \sqrt{p(1-p)}\|\pmb{M}_{J,J}^*\|_{2,\infty} + \sqrt{ps\sigma^2},
\\
R_3 
&:= \max \{ (1-p) \|\pmb{M}_{J^c,J}^*\|_{\max} + B,~~ p\|\pmb{M}_{J^c,J}^*\|_{\max} \},
\\
R_4 
&:= \max \{ \sqrt{p(1-p)}\|\pmb{M}_{J^c,J}^*\|_{2,\infty} + \sqrt{p(d-s)\sigma^2},
\sqrt{p(1-p)}\|\pmb{M}_{J,J^c}^*\|_{2,\infty} + \sqrt{ps\sigma^2} \},
\\
R_5 
&:= \max \{ (1-p) \|\pmb{M}_{J^c,J^c}^*\|_{\max} + B,~~ p\|\pmb{M}_{J^c,J^c}^*\|_{\max} \},
\\
R_6 
&:= \sqrt{p(1-p)}\|\pmb{M}_{J^c,J^c}^*\|_{2,\infty} + \sqrt{p(d-s)\sigma^2}.
\end{align*}
\end{lemma}

\begin{proof}
We use the matrix Bernstein inequality presented in Theorem \ref{thm:matrix_bernstein_ineq}.
Here, we only show the upper bound of $\| \mathbb{E}[\pmb{M}_{J,J}] - \pmb{M}_{J,J} \|_2$, since the others can be derived similarly.
Note that 
\begin{align*}
\mathbb{E}[\pmb{M}_{J,J}] - \pmb{M}_{J,J} 
&= \sum_{i,j\in J} (\mathbb{E}[M_{i,j}] - M_{i,j} )\pmb{e}_i \pmb{e}_j^\top
= \sum_{i,j\in J} (p\cdot M^*_{i,j} - \delta_{i,j}(M^*_{i,j} + \epsilon_{i,j}) )\pmb{e}_i \pmb{e}_j^\top
\\&=
\sum_{i,j\in J} 
(\underbrace{(p-\delta_{i,j})\cdot M^*_{i,j} - \delta_{i,j}\epsilon_{i,j}}_{=: a_{i,j}})
\pmb{e}_i \pmb{e}_j^\top
\\&=
\sum_{i,j \in J, i < j} 
a_{i,j} (\pmb{e}_i \pmb{e}_j^\top + \pmb{e}_j \pmb{e}_i^\top)
+
\sum_{i\in J} a_{i,i}\pmb{e}_i \pmb{e}_i^\top,
\end{align*}
which can be viewed as a sum of independent zero-mean matrices.
It is straightforward to conclude that
\begin{multline*}
\max \{ 
\|a_{i,j} (\pmb{e}_i \pmb{e}_j^\top + \pmb{e}_j \pmb{e}_i^\top)\|_2,~
\|a_{i,i}\pmb{e}_i \pmb{e}_i^\top\|_2 ~:~ i,j \in J,~ i < j\}
\\
\leq
\max \{ (1-p) \|\pmb{M}_{J,J}^*\|_{\max} + B,~~ p\|\pmb{M}_{J,J}^*\|_{\max} \}
= R_1 ~~a.s.,
\end{multline*}
\begin{equation*}
\Bigg\|
\sum_{i,j \in J, i < j} 
\mathbb{E}[
a_{i,j} (\pmb{e}_i \pmb{e}_j^\top + \pmb{e}_j \pmb{e}_i^\top)
]^2
+
\sum_{i\in J} 
\mathbb{E}[
a_{i,i}\pmb{e}_i \pmb{e}_i^\top ]^2
\Bigg\|_2
\leq
p(1-p)\|\pmb{M}_{J,J}^*\|_{2,\infty}^2 + ps\sigma^2 \leq R_2^2.
\end{equation*}
By the matrix Bernstein inequality, one has that with probability at least $1- (2s)^{-c}$,
$$
\| \mathbb{E}[\pmb{M}_{J,J}] - \pmb{M}_{J,J} \|_2 
\leq (c+1)\cdot R_1 \log (2s) + \sqrt{2(c+1)}\cdot R_2 \sqrt{\log (2s)}.
$$
\end{proof}

\begin{theorem}[Matrix Bernstein inequality (e.g., Theorem 1.6 in \cite{tropp2012user})]
\label{thm:matrix_bernstein_ineq}
Consider a finite sequence $\{\pmb{Z}_k\}$ of independent, random matrices with dimensions $d_1\times d_2$.
Assume that each random matrix satisfies
$$
\mathbb{E}[\pmb{Z}_k] = \pmb{0} ~~\text{ and }~~ \|\pmb{Z}_k\|_2 \leq R_1 ~\text{almost surely.}
$$
Also, suppose that
$$
\max\bigg\{    
\Big\| \sum_k \mathbb{E}[\pmb{Z}_k \pmb{Z}_k^\top] \Big\|_2,
\Big\| \sum_k \mathbb{E}[\pmb{Z}_k^\top \pmb{Z}_k] \Big\|_2
\bigg\} \leq R_2^2.
$$
Then, for all $t\geq 0$,
$$
\mathbb{P}\bigg\{\Big\|\sum_k \pmb{Z}_k \Big\|_2 \geq t \bigg\}
\leq (d_1+d_2)\cdot \exp \bigg( \frac{-t^2/2}{R_1 t/3 + R_2^2} \bigg).
$$
The above inequality implies that 
\begin{align*}
\Big\|\sum_k \pmb{Z}_k \Big\|_2 
&\leq 
\frac{(c+1)R_1}{3}\log(d_1+d_2)
+ \sqrt{ \bigg\{ \frac{(c+1)R_1}{3}\log(d_1+d_2)\bigg\}^2 + 2(c+1)R_2^2\log(d_1 + d_2)}
\\&\leq
(c+1)R_1 \log(d_1+d_2) + \sqrt{2(c+1)} R_2 \sqrt{\log(d_1 + d_2)}
\end{align*}
with probability at least $1-(d_1 + d_2)^{-c}$ for any $c>0$.
\end{theorem}

\subsection{Proof of Lemma \ref{lemma_suff_cond_sign}}
\label{subsec:lemma_suff_cond_sign}

First, the following lemma can be easily shown.

\begin{lemma}
\label{lemma_sign}
For any unit vectors $\pmb{x} \in \mathbb{R}^d$ and $\pmb{y} \in \mathbb{R}^d$ such that $y_i \neq 0$ for $\forall i\in[d]$, if $\|\pmb{x}-\pmb{y}\|_2 \leq \min_{i\in[d]} |y_i|$, then $sign(x_i) = sign(y_i)$ for $\forall i\in[d]$.
\end{lemma}

\begin{proof}
If $\pmb{x}=\pmb{y}$, then it is trivial that $sign(x_i) = sign(y_i)$ for $\forall i\in[d]$.
If $\pmb{x}\neq \pmb{y}$, then
for any $i\in[d]$,
$$
|x_i - y_i| < \|\pmb{x} - \pmb{y}\|_2 \leq \min_{i\in[d]} |y_i| \leq |y_i|,
$$
where the first inequality is strict since both $\pmb{x}$ and $\pmb{y}$ are unit vectors.
The above inequality implies that
$$
y_i - |y_i| < x_i < y_i + |y_i|,
$$
that is,
$0 < x_i < 2y_i$ if $y_i > 0$, and $2y_i < x_i < 0$ if $y_i < 0$.
Therefore, $sign(x_i) = sign(y_i)$ holds for any $i\in[d]$.
\end{proof}

Now, let $\hat{z}_i = \text{sign}(u_{1,i}) \text{~for all~} i\in J$
and
$\hat{\pmb{x}}$ be the leading eigenvector of $\pmb{M}_{J,J}-\rho \hat{\pmb{z}} \hat{\pmb{z}}^\top$.
We will derive the upper bound of $\|\pmb{u}_1 - \hat{\pmb{x}}\|_2$.
By applying the Davis-Kahan sin$\Theta$ theorem, we obtain
$$
\|\pmb{u}_1 - \hat{\pmb{x}}\|_2 \leq \frac{2\sqrt{2}}{p(\lambda_1(\pmb{M}^*_{J,J}) - \lambda_2(\pmb{M}^*_{J,J}))} \| \mathbb{E}[\pmb{M}_{J,J}] - \pmb{M}_{J,J} + \rho \hat{\pmb{z}} \hat{\pmb{z}}^\top\|_2
$$
where $\mathbb{E}[\pmb{M}_{J,J}] = p\cdot\pmb{M}^*_{J,J}$.
By the triangle inequality, we can upper bound
\begin{align*}
\| \mathbb{E}[\pmb{M}_{J,J}] - \pmb{M}_{J,J} + \rho \hat{\pmb{z}} \hat{\pmb{z}}^\top\|_2
\leq
\| \mathbb{E}[\pmb{M}_{J,J}] - \pmb{M}_{J,J} \|_2
+
\| \rho \hat{\pmb{z}} \hat{\pmb{z}}^\top\|_2
= \| \mathbb{E}[\pmb{M}_{J,J}] - \pmb{M}_{J,J} \|_2 + \rho s.
\end{align*}
From Lemma \ref{lemma_bernstein}, one has that with probability at least $1- (2s)^{-c}$,
$$
\|\pmb{u}_1 - \hat{\pmb{x}}\|_2 
\leq
2\sqrt{2}\cdot\frac{K_1 + \rho s}{p(\lambda_1(\pmb{M}^*_{J,J}) - \lambda_2(\pmb{M}^*_{J,J}))}.
$$
By Lemma \ref{lemma_sign}, if 
\begin{equation*}
2\sqrt{2}\cdot\frac{K_1 + \rho s}{p(\lambda_1(\pmb{M}^*_{J,J}) - \lambda_2(\pmb{M}^*_{J,J}))}
\leq
\min_{i \in J}|u_{1,i}|,
\end{equation*}
then $sign(\hat{x}_i) = sign(u_{1,i})$ for all $i\in J$ with probability at least $1- (2s)^{-c}$.

\subsection{Proof of Lemma \ref{lemma_suff_cond_w}}
\label{subsec:lemma_suff_cond_w}

First, we can derive the upper bound of $\|\hat{\pmb{w}}\|_\infty$ as follows:
\begin{align*}
\|\hat{\pmb{w}}\|_\infty
&= 
\frac{1}{\rho \|\hat{\pmb{x}}\|_1} \| \pmb{M}_{J^c, J}\hat{\pmb{x}} \|_\infty
= \frac{1}{\rho \|\hat{\pmb{x}}\|_1} \cdot \max_{i\in J^c} \bigg| \sum_{j\in J} M_{i,j} \hat{x}_j \bigg|
\\ &\leq
\frac{1}{\rho \|\hat{\pmb{x}}\|_1} \cdot 
\Big( \max_{i\in J^c} \max_{j\in J} |M_{i,j}| \Big)\cdot \sum_{j\in J} | \hat{x}_j |
= \frac{1}{\rho } \cdot \|\pmb{M}_{J^c,J}\|_{\max}. 
\end{align*}

For each $M_{i,j}$, $i\in J^c$ and $j\in J$, we now apply Chebyshev's inequality as follows:
\begin{align}
\label{ineq_cheby}
\mathbb{P} \big( |M_{i,j}| \geq |\mathbb{E}[M_{i,j}]| + c \big)
\leq
\mathbb{P} \big( |M_{i,j}-\mathbb{E}[M_{i,j}]| \geq c \big)
\leq
\frac{\mathsf{Var}[M_{i,j}]}{c^2}
\end{align}
for any $c>0$. Note that for each $i\in J^c$ and $j\in J$,
\begin{align*}
\mathbb{E}[M_{i,j}] = p\cdot M^*_{i,j} ~\text{ and }~
\mathsf{Var}[M_{i,j}] = p(1-p)(M^*_{i,j})^2 + p \sigma^2.
\end{align*}
With the assumption that $p\cdot \|\pmb{M}^*_{J^c,J}\|_{\max} < \rho$, 
letting $\gamma := \frac{p}{\rho} \|\pmb{M}^*_{J^c,J}\|_{\max}$ and $c := \frac{1+\gamma}{2}\cdot \rho - p\cdot |M^*_{i,j}| $
yields that $\gamma < 1$ and $c > 0$.
By plugging $c$ into (\ref{ineq_cheby}), we have that
\begin{align*}
\mathbb{P} \bigg( |M_{i,j}| \geq \frac{1+\gamma}{2}\cdot \rho \bigg)
&\leq
\frac{p(1-p)(M^*_{i,j})^2 + p \sigma^2}{\big(\frac{1+\gamma}{2}\cdot \rho - p\cdot |M^*_{i,j}|\big)^2}
\leq
\frac{p(1-p)(M^*_{i,j})^2 + p \sigma^2}{\big(\frac{1+\gamma}{2}\cdot \rho - p\cdot \|\pmb{M}^*_{J^c,J}\|_{\max}\big)^2}
\\ &=
\frac{p(1-p)(M^*_{i,j})^2 + p \sigma^2}{\big(\frac{1+\gamma}{2}\cdot \rho - \gamma \rho \big)^2}
= \frac{p(1-p)(M^*_{i,j})^2 + p \sigma^2}{\big(\frac{1-\gamma}{2}\cdot \rho\big)^2}
\end{align*}
for each $i\in J^c$ and $j\in J$. Hence,
\begin{align*}
\mathbb{P} \bigg( \|\pmb{M}_{J^c,J}\|_{\max} \geq \frac{1+\gamma}{2}\cdot \rho \bigg)
&\leq
\sum_{i\in J^c} \sum_{j\in J} \frac{p(1-p)(M^*_{i,j})^2 + p \sigma^2}{\big(\frac{1-\gamma}{2}\cdot \rho\big)^2}
\\&=
\frac{4}{(1-\gamma)^2 \rho^2} \cdot\big\{ p(1-p)\|\pmb{M}^*_{J^c,J}\|_F^2 + p(d-s)s\sigma^2 \big\}
\\&=
\frac{4}{\big(\rho-p\cdot \|\pmb{M}^*_{J^c,J}\|_{\max}\big)^2} \cdot\big\{ p(1-p)\|\pmb{M}^*_{J^c,J}\|_F^2 + p(d-s)s\sigma^2 \big\}.
\end{align*}
Since $\gamma < 1$, we have that
\begin{align*}
\mathbb{P} \bigg( \|\pmb{M}_{J^c,J}\|_{\max} < \rho \bigg)
\geq
1 - \frac{4}{\big(\rho-p\cdot \|\pmb{M}^*_{J^c,J}\|_{\max}\big)^2} 
\cdot\big\{ p(1-p)\|\pmb{M}^*_{J^c,J}\|_F^2 + p(d-s)s\sigma^2 \big\}.
\end{align*}
If the following inequality holds:
\begin{align*}
\sqrt{4ps^c  \cdot \big\{(1-p) \|\pmb{M}^*_{J^c,J}\|_F^2 + (d-s)s\sigma^2 \big\} } < \rho-p\cdot \|\pmb{M}^*_{J^c,J}\|_{\max},
\end{align*}
then 
$\frac{4}{\big(\rho-p\cdot \|\pmb{M}^*_{J^c,J}\|_{\max}\big)^2} 
\cdot\big\{ p(1-p)\|\pmb{M}^*_{J^c,J}\|_F^2 + p(d-s)s\sigma^2 \big\} \leq s^{-c}$,
that is,
$\|\pmb{M}_{J^c,J}\|_{\max} < \rho$ holds with probability at least $1-s^{-c}$.
$\|\hat{\pmb{w}}\|_\infty \leq \frac{1}{\rho}\|\pmb{M}_{J^c,J}\|_{\max}$, and thus the desired result holds.

\subsection{Proof of Lemma \ref{lemma_suff_cond_eig}}
\label{subsec:lemma_suff_cond_eig}

Lemma \ref{lemma_eig_ineq} shows that 
if the following inequality holds:
$$
\| \pmb{M}_{J^c,J}-\rho \hat{\pmb{w}} \hat{\pmb{z}}^\top \|_2^2 
\leq 
\big\{ \lambda_1(\pmb{M}_{J,J}-\rho \hat{\pmb{z}} \hat{\pmb{z}}^\top) - \lambda_2(\pmb{M}_{J,J}-\rho \hat{\pmb{z}} \hat{\pmb{z}}^\top) \big\} 
\cdot \big\{\lambda_1(\pmb{M}_{J,J}-\rho \hat{\pmb{z}} \hat{\pmb{z}}^\top) 
- \lambda_1( \pmb{M}_{J^c,J^c}-\rho \hat{\pmb{w}} \hat{\pmb{w}}^\top ) \big\},
$$
then
$\lambda_1(\pmb{M}_{J,J}-\rho \hat{\pmb{z}} \hat{\pmb{z}}^\top) = \lambda_1(\pmb{M}-\rho (\hat{\pmb{z}}^\top, \hat{\pmb{w}}^\top)^\top (\hat{\pmb{z}}^\top, \hat{\pmb{w}}^\top))$.

Now, we derive the upper or lower bounds of 
$\| \pmb{M}_{J^c,J}-\rho \hat{\pmb{w}} \hat{\pmb{z}}^\top \|_2$,
$\lambda_1(\pmb{M}_{J,J}-\rho \hat{\pmb{z}} \hat{\pmb{z}}^\top) - \lambda_2(\pmb{M}_{J,J}-\rho \hat{\pmb{z}} \hat{\pmb{z}}^\top)$
and 
$\lambda_1(\pmb{M}_{J,J}-\rho \hat{\pmb{z}} \hat{\pmb{z}}^\top) 
- \lambda_1( \pmb{M}_{J^c,J^c}-\rho \hat{\pmb{w}} \hat{\pmb{w}}^\top )$.
First,
\begin{align*}
\| \pmb{M}_{J^c,J}-\rho \hat{\pmb{w}} \hat{\pmb{z}}^\top \|_2
&=
\bigg\| \pmb{M}_{J^c,J}-\rho\cdot \frac{1}{\rho \|\hat{\pmb{x}}\|_1} \pmb{M}_{J^c, J}\hat{\pmb{x}} \hat{\pmb{z}}^\top \bigg\|_2
=
\bigg\| \pmb{M}_{J^c,J} \cdot \bigg(I-\frac{\hat{\pmb{x}} \hat{\pmb{z}}^\top}{\|\hat{\pmb{x}}\|_1}  \bigg) \bigg\|_2
\\ &\leq
\| \pmb{M}_{J^c,J} \|_2 \cdot \bigg\| I-\frac{\hat{\pmb{x}} \hat{\pmb{z}}^\top}{\|\hat{\pmb{x}}\|_1}  \bigg\|_2
\leq
\| \pmb{M}_{J^c,J} \|_2 \cdot \bigg(1+\frac{\|\hat{\pmb{x}}\|_2 \|\hat{\pmb{z}}\|_2}{\|\hat{\pmb{x}}\|_1}  \bigg)
\\ &\leq
\| \pmb{M}_{J^c,J} \|_2 \cdot (1+\sqrt{s}).
\end{align*}
Since 
$
\| \pmb{M}_{J^c,J} \|_2
\leq
\| \mathbb{E}[\pmb{M}_{J^c,J}] - \pmb{M}_{J^c,J} \|_2
+ \| \mathbb{E}[\pmb{M}_{J^c,J}] \|_2
\leq
K_2 + p\cdot \| \pmb{M}^*_{J^c,J} \|_2
$
with probability at least $1-d^{-c}$ by Lemma \ref{lemma_bernstein},
we have that for any $c>0$,
\begin{align}
\label{ineq_bound1}
\| \pmb{M}_{J^c,J}-\rho \hat{\pmb{w}} \hat{\pmb{z}}^\top \|_2
\leq 
(K_2 + p\cdot \| \pmb{M}^*_{J^c,J} \|_2)\cdot (1+\sqrt{s})
\end{align}
with probability at least $1-d^{-c}$.

Next, by Weyl's inequality,
\begin{align*}
\lambda_1(\pmb{M}_{J,J}-\rho \hat{\pmb{z}} \hat{\pmb{z}}^\top) - \lambda_2(\pmb{M}_{J,J}-\rho \hat{\pmb{z}} \hat{\pmb{z}}^\top)
&\geq
\lambda_1(\mathbb{E}[\pmb{M}_{J,J}]) - \lambda_2(\mathbb{E}[\pmb{M}_{J,J}])
-
2\cdot\| \mathbb{E}[\pmb{M}_{J,J}] - \pmb{M}_{J,J} + \rho \hat{\pmb{z}} \hat{\pmb{z}}^\top \|_2
\\ &\geq
p\cdot (\lambda_1(\pmb{M}^*_{J,J}) - \lambda_2(\pmb{M}^*_{J,J}))
-
2\cdot\| \mathbb{E}[\pmb{M}_{J,J}] - \pmb{M}_{J,J}  \|_2 - 2\rho s.
\end{align*}
By Lemma \ref{lemma_bernstein}, we have that for any $c>0$,
\begin{align}
\label{ineq_bound2}
\lambda_1(\pmb{M}_{J,J}-\rho \hat{\pmb{z}} \hat{\pmb{z}}^\top) - \lambda_2(\pmb{M}_{J,J}-\rho \hat{\pmb{z}} \hat{\pmb{z}}^\top)
\geq
p\cdot (\lambda_1(\pmb{M}^*_{J,J}) - \lambda_2(\pmb{M}^*_{J,J}))
-
2\cdot K_1 - 2\rho s
\end{align}
with probability at least $1-(2s)^{-c}$.

Finally, by applying Weyl's inequality and the triangle inequality, we have that
\begin{align*}
\lambda_1(\pmb{M}_{J,J}-\rho \hat{\pmb{z}} \hat{\pmb{z}}^\top) 
- \lambda_1( \pmb{M}_{J^c,J^c}-\rho \hat{\pmb{w}} \hat{\pmb{w}}^\top )
&\geq 
p\cdot \lambda_1(\pmb{M}^*_{J,J})
- \| \mathbb{E}[\pmb{M}_{J,J}] - \pmb{M}_{J,J}  \|_2 - \rho s
\\ &~~
- \| \mathbb{E}[\pmb{M}_{J^c,J^c}] - \pmb{M}_{J^c,J^c}  \|_2
- p\cdot\lambda_1(\pmb{M}^*_{J^c,J^c})
- \rho \cdot \| \hat{\pmb{w}}\|_2^2.
\end{align*}
Note that under the conditions in Lemma \ref{lemma_suff_cond_w}, 
$\|\hat{\pmb{w}}\|_\infty < 1$, that is,
$\|\hat{\pmb{w}}\|_2 < \sqrt{d-s}$ holds with probability at least $1-s^{-c}$.
We can also use the upper bound of $\|\hat{\pmb{w}}\|_2$ derived in Lemma \ref{lemma_w_2}.
By applying Lemma \ref{lemma_bernstein}, we have that for any $c>0$,
\begin{align}
\label{ineq_bound3}
\lambda_1(\pmb{M}_{J,J}-\rho \hat{\pmb{z}} \hat{\pmb{z}}^\top) 
- \| \pmb{M}_{J^c,J^c}-\rho \hat{\pmb{w}} \hat{\pmb{w}}^\top \|_2
&\geq 
p\cdot (\lambda_1(\pmb{M}^*_{J,J}) - \lambda_1(\pmb{M}^*_{J^c,J^c}))
- K_1 - K_3 - \rho d
\end{align}
or
\begin{align}
\lambda_1(\pmb{M}_{J,J}-\rho \hat{\pmb{z}} \hat{\pmb{z}}^\top) 
- \| \pmb{M}_{J^c,J^c}-\rho \hat{\pmb{w}} \hat{\pmb{w}}^\top \|_2
&\geq 
p\cdot (\lambda_1(\pmb{M}^*_{J,J}) - \lambda_1(\pmb{M}^*_{J^c,J^c}))
- K_1 - \rho s 
\nonumber
\\&~~~~ - K_3 - \frac{1}{\rho} \cdot 
( p\|\pmb{M}^*_{J,J^c} \|_{\infty, 2} + c_0\sqrt{d-s} )^2
\label{eq:using_another_bound_w2}
\end{align}
with probability at least $1-s^{-c} - (2s)^{-c} - (2(d-s))^{-c}$,
where $c_0 = \sqrt{s^c \cdot \{ p(1-p)\|\pmb{M}^*_{J^c,J}\|_F^2 + p s(d-s)\sigma^2 \}}$.

From (\ref{ineq_bound1})-(\ref{ineq_bound3}), we can derive that if the following inequality holds:
\begin{align*}
(K_2 + p\cdot \| \pmb{M}^*_{J^c,J} \|_2)^2\cdot (1+\sqrt{s})^2
&\leq
\Big\{ 
p\cdot (\lambda_1(\pmb{M}^*_{J,J}) - \lambda_2(\pmb{M}^*_{J,J}))
-
2\cdot K_1 - 2\rho s \Big\}
\\&~~~~\times
\Big\{
p\cdot (\lambda_1(\pmb{M}^*_{J,J}) - \| \pmb{M}^*_{J^c,J^c} \|_2)
- K_1 - K_3 - \rho d\Big\},
\end{align*}
then for any $c>0$, the desired result holds with probability at least $1-s^{-c} - d^{-c} - (2s)^{-c} - (2(d-s))^{-c}$.

\begin{lemma}
\label{lemma_eig_ineq}
If the following inequality holds:
$$
\| \pmb{M}_{J^c,J}-\rho \hat{\pmb{w}} \hat{\pmb{z}}^\top \|_2^2 
\leq 
\big\{ \lambda_1(\pmb{M}_{J,J}-\rho \hat{\pmb{z}} \hat{\pmb{z}}^\top) - \lambda_2(\pmb{M}_{J,J}-\rho \hat{\pmb{z}} \hat{\pmb{z}}^\top) \big\} 
\cdot \big\{\lambda_1(\pmb{M}_{J,J}-\rho \hat{\pmb{z}} \hat{\pmb{z}}^\top) 
- \| \pmb{M}_{J^c,J^c}-\rho \hat{\pmb{w}} \hat{\pmb{w}}^\top \|_2 \big\},
$$
then
$\lambda_1(\pmb{M}_{J,J}-\rho \hat{\pmb{z}} \hat{\pmb{z}}^\top) = \lambda_1(\pmb{M}-\rho (\hat{\pmb{z}}^\top, \hat{\pmb{w}}^\top)^\top (\hat{\pmb{z}}^\top, \hat{\pmb{w}}^\top))$.
\end{lemma}

\begin{proof}
Let $\hat{\pmb{Z}} = \begin{pmatrix}
\hat{\pmb{z}} \hat{\pmb{z}}^\top & \hat{\pmb{z}} \hat{\pmb{w}}^\top \\
\hat{\pmb{w}} \hat{\pmb{z}}^\top & \hat{\pmb{w}} \hat{\pmb{w}}^\top
\end{pmatrix}$. 
First, we can show that 
$\lambda_1(\pmb{M}_{J,J}-\rho \hat{\pmb{z}} \hat{\pmb{z}}^\top)$ 
is an eigenvalue of $\pmb{M}-\rho \hat{\pmb{Z}}$
where its corresponding eigenvector is 
$(\hat{\pmb{x}}^\top, 0^\top)^\top \in \mathbb{R}^d$.
This is because
$$
(\pmb{M}-\rho \hat{\pmb{Z}})\begin{pmatrix} \hat{\pmb{x}} \\ 0 \end{pmatrix}
=
\begin{pmatrix}
(\pmb{M}_{J,J}-\rho \hat{\pmb{z}} \hat{\pmb{z}}^\top) \hat{\pmb{x}} \\
(\pmb{M}_{J^c,J}-\rho \hat{\pmb{w}} \hat{\pmb{z}}^\top) \hat{\pmb{x}}
\end{pmatrix}
= \lambda_1(\pmb{M}_{J,J}-\rho \hat{\pmb{z}} \hat{\pmb{z}}^\top) \cdot
\begin{pmatrix} \hat{\pmb{x}} \\ 0 \end{pmatrix}
$$
where the last equality holds since $\hat{\pmb{x}}$ is the leading eigenvector of $\pmb{M}_{J,J}-\rho \hat{\pmb{z}} \hat{\pmb{z}}^\top$ and
\begin{align*}
(\pmb{M}_{J^c,J}-\rho \hat{\pmb{w}} \hat{\pmb{z}}^\top) \hat{\pmb{x}}
= 
\pmb{M}_{J^c,J} \hat{\pmb{x}}
- \rho \cdot \frac{1}{\rho \|\hat{\pmb{x}}\|_1} \pmb{M}_{J^c, J}\hat{\pmb{x}} 
\cdot \|\hat{\pmb{x}}\|_1
= 0.
\end{align*}

Now, it is sufficient to show that for all $\pmb{y} = (\pmb{y}_1^\top,  ~\pmb{y}_2^\top )^\top$ such that $\pmb{y}_1 \in \mathbb{R}^s$, $\pmb{y}_2 \in \mathbb{R}^{d-s}$, $\|\pmb{y}_1\|_2^2 + \|\pmb{y}_2\|_2^2 = 1$ and $\hat{\pmb{x}}^\top \pmb{y_1} = 0$,
$$
\pmb{y}^\top (\pmb{M}-\rho \hat{\pmb{Z}}) \pmb{y} 
\leq 
\lambda_1(\pmb{M}_{J,J}-\rho \hat{\pmb{z}} \hat{\pmb{z}}^\top),
$$
which implies that $\lambda_1(\pmb{M}_{J,J}-\rho \hat{\pmb{z}} \hat{\pmb{z}}^\top)$ is the largest eigenvalue of $\pmb{M}-\rho \hat{\pmb{Z}}$.
Note that 
\begin{align*}
&\pmb{y}^\top (\pmb{M}-\rho \hat{\pmb{Z}}) \pmb{y}
= \pmb{y}_1^\top (\pmb{M}_{J,J}-\rho \hat{\pmb{z}} \hat{\pmb{z}}^\top) \pmb{y}_1
+ 
2\pmb{y}_2^\top (\pmb{M}_{J^c,J}-\rho \hat{\pmb{w}} \hat{\pmb{z}}^\top) \pmb{y}_1
+ 
\pmb{y}_2^\top (\pmb{M}_{J^c,J^c}-\rho \hat{\pmb{w}} \hat{\pmb{w}}^\top) \pmb{y}_2
\\ &\leq
\lambda_2(\pmb{M}_{J,J}-\rho \hat{\pmb{z}} \hat{\pmb{z}}^\top)\cdot \|\pmb{y}_1\|_2^2
+ 2 \| \pmb{M}_{J^c,J}-\rho \hat{\pmb{w}} \hat{\pmb{z}}^\top \|_2\cdot \|\pmb{y}_1\|_2\cdot \|\pmb{y}_2\|_2
+ \lambda_1( \pmb{M}_{J^c,J^c}-\rho \hat{\pmb{w}} \hat{\pmb{w}}^\top )\cdot \|\pmb{y}_2\|_2^2
\\ &=
\lambda_2(\pmb{M}_{J,J}-\rho \hat{\pmb{z}} \hat{\pmb{z}}^\top)\cdot 
(1-\|\pmb{y}_2\|_2^2)
+ 2 \| \pmb{M}_{J^c,J}-\rho \hat{\pmb{w}} \hat{\pmb{z}}^\top \|_2 \cdot 
\sqrt{1-\|\pmb{y}_2\|_2^2} \cdot \|\pmb{y}_2\|_2
+ \lambda_1( \pmb{M}_{J^c,J^c}-\rho \hat{\pmb{w}} \hat{\pmb{w}}^\top )\cdot \|\pmb{y}_2\|_2^2
\\ &=
\lambda_2(\pmb{M}_{J,J}-\rho \hat{\pmb{z}} \hat{\pmb{z}}^\top)
+ (\lambda_1( \pmb{M}_{J^c,J^c}-\rho \hat{\pmb{w}} \hat{\pmb{w}}^\top ) - \lambda_2(\pmb{M}_{J,J}-\rho \hat{\pmb{z}} \hat{\pmb{z}}^\top)) \cdot \|\pmb{y}_2\|_2^2
\\ &~~~+ 2 \| \pmb{M}_{J^c,J}-\rho \hat{\pmb{w}} \hat{\pmb{z}}^\top \|_2 \cdot 
\sqrt{\|\pmb{y}_2\|_2^2\cdot(1-\|\pmb{y}_2\|_2^2)}
\\ &=
\lambda_2(\pmb{M}_{J,J}-\rho \hat{\pmb{z}} \hat{\pmb{z}}^\top)
+ (\lambda_1( \pmb{M}_{J^c,J^c}-\rho \hat{\pmb{w}} \hat{\pmb{w}}^\top ) - \lambda_2(\pmb{M}_{J,J}-\rho \hat{\pmb{z}} \hat{\pmb{z}}^\top)) \cdot t
+ 2 \| \pmb{M}_{J^c,J}-\rho \hat{\pmb{w}} \hat{\pmb{z}}^\top \|_2 \cdot 
\sqrt{t \cdot(1-t)}
\end{align*}
where $0\leq t := \|\pmb{y}_2\|_2^2 \leq 1$.
The first inequality holds since $\pmb{y}_1/\|\pmb{y}_1\|_2$ is orthogonal to $\hat{\pmb{x}}$, the leading eigenvector of $\pmb{M}_{J,J}-\rho \hat{\pmb{z}} \hat{\pmb{z}}^\top$.
The above upper bound of $\pmb{y}^\top (\pmb{M}-\rho \hat{\pmb{Z}}) \pmb{y}$ implies that if the following inequality holds for any $t\in [0,1]$:
\begin{multline*}
\lambda_2(\pmb{M}_{J,J}-\rho \hat{\pmb{z}} \hat{\pmb{z}}^\top)
+ (\lambda_1( \pmb{M}_{J^c,J^c}-\rho \hat{\pmb{w}} \hat{\pmb{w}}^\top ) - \lambda_2(\pmb{M}_{J,J}-\rho \hat{\pmb{z}} \hat{\pmb{z}}^\top)) \cdot t
+ 2 \| \pmb{M}_{J^c,J}-\rho \hat{\pmb{w}} \hat{\pmb{z}}^\top \|_2 \cdot 
\sqrt{t \cdot(1-t)}
\\ \leq
\lambda_1(\pmb{M}_{J,J}-\rho \hat{\pmb{z}} \hat{\pmb{z}}^\top),
\end{multline*}
then $\lambda_1(\pmb{M}_{J,J}-\rho \hat{\pmb{z}} \hat{\pmb{z}}^\top)$ is the largest eigenvalue of $\pmb{M}-\rho \hat{\pmb{Z}}$.
From Lemma \ref{lemma_quad}, we have that if the following inequality holds:
$$
\| \pmb{M}_{J^c,J}-\rho \hat{\pmb{w}} \hat{\pmb{z}}^\top \|_2^2 
\leq 
\big\{ \lambda_1(\pmb{M}_{J,J}-\rho \hat{\pmb{z}} \hat{\pmb{z}}^\top) - \lambda_2(\pmb{M}_{J,J}-\rho \hat{\pmb{z}} \hat{\pmb{z}}^\top) \big\} 
\cdot \big\{\lambda_1(\pmb{M}_{J,J}-\rho \hat{\pmb{z}} \hat{\pmb{z}}^\top) 
- \lambda_1( \pmb{M}_{J^c,J^c}-\rho \hat{\pmb{w}} \hat{\pmb{w}}^\top ) \big\},
$$
then
$\lambda_1(\pmb{M}_{J,J}-\rho \hat{\pmb{z}} \hat{\pmb{z}}^\top) = \lambda_1(\pmb{M}-\rho \hat{\pmb{Z}})$.
\end{proof}

\begin{lemma}
\label{lemma_w_2}
Let $c_0 = \sqrt{s^c \cdot \{ p(1-p)\|\pmb{M}^*_{J^c,J}\|_F^2 + p s(d-s)\sigma^2 \}}$ for any $c>0$. Then,
$$
\|\hat{\pmb{w}}\|_2 
\leq
\frac{p}{\rho}\cdot \|\pmb{M}^*_{J,J^c} \|_{\infty, 2} + \frac{c_0\sqrt{d-s}}{\rho}
$$
with probability at least $1 - s^{-c}$.
\end{lemma}

\begin{proof}
First, we can derive the upper bound of $\|\hat{\pmb{w}}\|_2$ as follows:
\begin{align*}
\|\hat{\pmb{w}}\|_2
&= 
\frac{1}{\rho \|\hat{\pmb{x}}\|_1} \| \pmb{M}_{J^c, J}\hat{\pmb{x}} \|_2
= \frac{1}{\rho \|\hat{\pmb{x}}\|_1} \cdot \sqrt{ \sum_{i\in J^c} \bigg( \sum_{j\in J} M_{i,j} \hat{x}_j \bigg)^2}
\\ &\leq
\frac{1}{\rho \|\hat{\pmb{x}}\|_1} \cdot \sqrt{ \sum_{i\in J^c} 
\Big(\max_{j\in J} |M_{i,j}|\cdot \|\hat{\pmb{x}}\|_1 \Big)^2}
= \frac{1}{\rho } \cdot \sqrt{ \sum_{i\in J^c} 
\Big(\max_{j\in J} |M_{i,j}| \Big)^2}.
\end{align*}
By Chebyshev's inequality, for each $i\in J^c$ and $j\in J$,
\begin{align}
\mathbb{P} \big( |M_{i,j}| \geq |\mathbb{E}[M_{i,j}]| + c_0 \big)
\leq
\frac{\mathsf{Var}[M_{i,j}]}{c_0^2}
\end{align}
for any $c_0>0$, where
$\mathbb{E}[M_{i,j}] = p\cdot M^*_{i,j}$ and
$\mathsf{Var}[M_{i,j}] = p(1-p)(M^*_{i,j})^2 + p \sigma^2.$
That is,
\begin{align*}
&\mathbb{P} \big( |M_{i,j}| \leq pM^*_{i,j} + c_0 
\text{ for } \forall i\in J^c \text{ and } j\in J\big)
\\ &\geq
1 - \sum_{i\in J^c, j \in J}\mathbb{P} \big( |M_{i,j}| \geq |\mathbb{E}[M_{i,j}]| + c_0 \big)
\\ &\geq
1 - \frac{p(1-p)\|\pmb{M}^*_{J^c,J}\|_F^2 + p s(d-s)\sigma^2}{c_0^2}.
\end{align*}
Let $c_0 = \sqrt{s^c \cdot \{ p(1-p)\|\pmb{M}^*_{J^c,J}\|_F^2 + p s(d-s)\sigma^2 \}}$ for any $c>0$. 
Then $|M_{i,j}| \leq pM^*_{i,j} + c_0$ for any $i\in J^c$ and $j\in J$ with probability at least $1 - s^{-c}$.
This also means that $\max_{j\in J} |M_{i,j}| \leq p \max_{j\in J} |M_{i,j}^*| + c_0$ and
$$
\sqrt{ \sum_{i\in J^c} \Big(\max_{j\in J} |M_{i,j}| \Big)^2}
\leq
\sqrt{ \sum_{i\in J^c} 
\Big( p \max_{j\in J} |M_{i,j}^*| \Big)^2}
+ \sqrt{\sum_{i\in J^c} c_0^2}
$$
with probability at least $1 - s^{-c}$.
Therefore, for $c_0 = \sqrt{s^c \cdot \{ p(1-p)\|\pmb{M}^*_{J^c,J}\|_F^2 + p s(d-s)\sigma^2 \}}$,
$$
\|\hat{\pmb{w}}\|_2 
\leq
\frac{p}{\rho}\cdot \|\pmb{M}^*_{J,J^c} \|_{\infty, 2} + \frac{c_0\sqrt{d-s}}{\rho}
$$
with probability at least $1 - s^{-c}$ for any $c>0$, where $\|\pmb{M}^*_{J,J^c} \|_{\infty, 2} = \sqrt{ \sum_{i\in J^c} 
\Big( \max_{j\in J} |M^*_{i,j}| \Big)^2}$.
\end{proof}

\begin{lemma}
\label{lemma_quad}
Assume $a\neq 0$. If $a^2 \leq c(b+c)$ holds,
then $2a\sqrt{t(1-t)} \leq bt + c \text{ for all } t\in[0,1]$.
\end{lemma}

\begin{proof}
\begin{align*}
&2a\sqrt{t(1-t)} \leq bt + c ~~~\text{ for all } t\in[0,1]
\\ \Leftarrow&~
4a^2 t(1-t) \leq (bt + c)^2,~ bt + c \geq 0 ~~~\text{ for all } t\in[0,1]
\\ \Leftrightarrow&~
(4a^2 + b^2)\bigg(t - \frac{2a^2 - bc}{4a^2 + b^2}\bigg)^2 + c^2 - \frac{(2a^2 - bc)^2}{4a^2 + b^2} \geq 0,~ bt + c \geq 0 ~~~\text{ for all } t\in[0,1] 
\\ \Leftarrow&~
c^2 - \frac{(2a^2 - bc)^2}{4a^2 + b^2} \geq 0,~ c \geq 0,~ b+c \geq 0
\\ \Leftrightarrow&~
a^2 \leq c(b+c).
\end{align*}
\end{proof}

\subsection{Proof of Lemma \ref{lemma_suff_cond_strict_eig}}
\label{subsec:lemma_suff_cond_strict_eig}

By Weyl's inequality and the triangle inequality,
\begin{align*}
&\lambda_1(\pmb{M}_{J,J}-\rho \hat{\pmb{z}} \hat{\pmb{z}}^\top) 
- \lambda_2(\pmb{M}_{J,J}-\rho \hat{\pmb{z}} \hat{\pmb{z}}^\top)
\\ &\geq
\lambda_1(\mathbb{E}[\pmb{M}_{J,J}]) - \lambda_2(\mathbb{E}[\pmb{M}_{J,J}])
- 2\cdot \| \mathbb{E}[\pmb{M}_{J,J}] - \pmb{M}_{J,J} \|_2 
- 2\rho \cdot \|\hat{\pmb{z}} \hat{\pmb{z}}^\top \|_2
\\ &=
p\cdot (\lambda_1(\pmb{M}^*_{J,J})-\lambda_2(\pmb{M}^*_{J,J})) 
- 2\cdot \| \mathbb{E}[\pmb{M}_{J,J}] - \pmb{M}_{J,J} \|_2 
- 2\rho \cdot \|\hat{\pmb{z}}\|_2^2.
\end{align*}
Note that $\|\hat{\pmb{z}}\|_2 = \sqrt{s}$ and under the conditions in Lemma \ref{lemma_suff_cond_w}.
Also, by applying Lemma \ref{lemma_bernstein}, we have that
$$
\lambda_1(\pmb{M}_{J,J}-\rho \hat{\pmb{z}} \hat{\pmb{z}}^\top) 
- \lambda_2(\pmb{M}_{J,J}-\rho \hat{\pmb{z}} \hat{\pmb{z}}^\top)
>
p\cdot (\lambda_1(\pmb{M}^*_{J,J})-\lambda_2(\pmb{M}^*_{J,J})) -2\cdot K_1 -2\rho s
$$
with probability at least $1-(2s)^{-c}$ for any $c>0$.
Therefore, the desired result holds.

\end{document}